\documentclass[twoside,11pt]{article}
\usepackage{jmlr2e}

\hypersetup{
colorlinks=false,
hidelinks
}
\usepackage[utf8]{inputenc} 
\usepackage[T1]{fontenc}    
\usepackage{url}            
\usepackage{booktabs}       
\usepackage{amsfonts}       
\usepackage{nicefrac}       
\usepackage{microtype}      
\usepackage{xcolor}         
\usepackage{amssymb}
\usepackage{natbib}
\usepackage{color}
\usepackage{url}
\usepackage{amsmath}
\usepackage{wrapfig}
\usepackage{mathtools}
\usepackage{tikz}
\usepackage{algorithmic}
\usepackage{footmisc}
\usepackage{bibentry}
\nobibliography*

\usepackage{physics}
\usepackage{mathrsfs}
\mathtoolsset{showonlyrefs}
\usepackage{etoolbox}
\usepackage{makecell}
\usepackage[ruled]{algorithm2e}
\DontPrintSemicolon
\usepackage{enumerate}
\usepackage{lastpage}
\usepackage{float}
\usepackage{thmtools, thm-restate}
\usepackage[normalem]{ulem}

\newcommand{\explain}[1]{\tag*{(#1)}}

\newcommand{\fS}{\mathcal{S}}

\newcommand{\fY}{\mathcal{Y}}
\newcommand{\fB}{\mathcal{B}}
\newcommand{\fZ}{\mathcal{Z}}
\newcommand{\fW}{\mathcal{W}}
\newcommand{\fF}{\mathcal{F}}

\newcommand{\mi}{\texttt{i}}

\newcommand{\R}[1][]{\mathbb{R}^{#1}}
\newcommand{\C}[1][]{\mathbb{C}^{#1}}

\newcommand{\E}{\mathbb{E}}

\newcommand{\ns}{{|\fS|}}

\newcommand{\bop}{\mathcal{T}}

\newcommand{\indot}[2]{{\left<#1, #2\right>}}

\newcommand{\tref}[1]{\text{\ref{#1}}}

\newcommand{\Dnorm}[1]{\norm{#1}_D}
\newcommand{\innerdot}[2]{\left<{#1}, {#2}\right>}

\newtheorem{xtheorem}{Theorem}[section]
\newtheorem{xdefinition}{Definition}[section]
\newtheorem{xlemma}{Lemma}[section]

\newtheorem{xassumption}{Assumption}[section]

\allowdisplaybreaks
\allowbreak

\begin{document}

\title{Almost Sure Convergence of Linear Temporal Difference Learning with Arbitrary Features}

%

\author{\name Jiuqi Wang \email jiuqi@email.virginia.edu \\
\addr 
Department of Computer Science \\
University of Virginia\\
85 Engineer's Way, Charlottesville, VA, 22903
\AND
\name Shangtong Zhang \email shangtong@virginia.edu \\
\addr 
Department of Computer Science,  University of Virginia\\
University of Virginia \\
85 Engineer's Way, Charlottesville, VA, 22903 \\
}

\jmlrheading{27}{2026}{1-\pageref{LastPage}}{9/24; Revised 12/25}{3/26}{24-1538}{Jiuqi Wang, Shangtong Zhang}
\ShortHeadings{Convergence of Linear TD with Arbitrary Features}{Wang and Zhang}

\firstpageno{1}

\editor{Laurent Orseau}

\maketitle

\begin{abstract}
  Temporal difference (TD) learning with linear function approximation (linear TD) is a classic and powerful prediction algorithm in reinforcement learning. 
  While it is well-understood that linear TD converges almost surely to a unique point, this convergence traditionally requires the assumption that the features used by the approximator are linearly independent.  
  However, this linear independence assumption does not hold in many practical scenarios. This work is the first to establish the almost sure convergence of linear TD without requiring linearly independent features. 
  We prove that the weight iterates of linear TD converge to a bounded set, and that the value estimates derived from the weights in that set are the same almost everywhere. 
  We also establish a notion of local stability of the weight iterates. 
  Importantly, we do not impose assumptions tailored to feature dependence and do not modify the linear TD algorithm. 
  Key to our analysis is a novel characterization of bounded invariant sets of the mean ODE of linear TD. 
\end{abstract}



\section{Introduction}
Function approximation is crucial in reinforcement learning (RL) algorithms when the problem involves an intractable discrete or continuous state space~\citep{sutton2018reinforcement}. 
The idea is to encode the states into finite-dimensional real-valued vectors called features. 
A parameterized function of the features, with learnable weights, is then used to approximate the desired function. 
For instance, linear function approximation computes the approximated value by taking the dot product of the feature and the weight. 
Most existing convergence results with linear function approximation assume the features are linearly independent~\citep{tsitsiklis1997analysis,konda2000actor,sutton2008gtd,sutton2009fast, maei2011gradient, yu2015convergence, sutton2016emphatic, lee2019target, nachum2019dualdice, zou2019finite, carvalho2020new, zhang2020gradientdice, zhang2020provably, zhang2021average, zhang2021breaking,zhang2022truncated, zhang2023convergence, qian2025revisiting}.  
However, the linear independence assumption is not desired for at least four reasons. 
First, many well-known empirical successes of RL with linear function approximation~\citep{liang2016shallow, azagirre2024lyft} do not have linearly independent features, leaving a gap between theory and practice. 
Second, in the continual learning setting~\citep{ring1994continual, khetarpal2022towards,abel2023definition}, the observations received by an agent are usually served one after another. 
There is no way to verify whether the features used in the observations are linearly independent. 
Third, the features are sometimes constructed via neural networks \citep{chung2018two}. 
Usually, it is impossible to guarantee that those neural network-based features are linearly independent. 
Fourth, sometimes the features are gradients of another neural network, e.g., in the compatible feature framework for actor-critic algorithms (see \citet{sutton2000policy,konda2000actor,zhang2020provably} for details). 
One cannot guarantee those features are linearly independent either. 
As a result, although the linear independence assumption greatly simplifies the theoretical analysis, it is unrealistically restrictive. 
\citet{dayan1992tdlambda} and~\citet{tsitsiklis1997analysis,tsitsiklis1999average} also identify the removal of the linear independence assumption of the features as a future research direction.

This work makes progress towards closing this gap using linear temporal difference (TD) learning~\citep{sutton1988learning} as an example, since linear TD is arguably one of the most fundamental RL algorithms. 
In particular, this work is the first to establish the almost sure convergence of linear TD without requiring linearly independent features. 
The main contributions of this work are
\begin{enumerate}
  \item characterization of TD fixed points under arbitrary features and proof that all such fixed points produce the same value estimate almost everywhere;
  \item mean ODE analysis with Jordan normal form demonstrating every ODE trajectory converges to an initial-condition-dependent point;
  \item characterization of bounded invariant sets of the mean ODE in the absence of global asymptotic stability;
  \item establishment of a notion of local stability of the weight iterates.
\end{enumerate}
Importantly, our proof does not introduce assumptions tailored to feature dependence and keep linear TD in its original form. 
Building upon our work,~\citet{xie2025finite} conduct finite sample analysis of linear TD under arbitrary features to characterize its convergence rate, further completing the understanding of linear TD in its canonical form.

\section{Background}
\label{section: background}
\paragraph*{Notations.} A square complex matrix (not necessarily symmetric) $M \in \C[n\times n]$ is said to be positive definite if, for any non-zero vector $x \in \C[n]$, it holds that $\Re(x^H M x) > 0$, where $x^H$ denotes the conjugate transpose of $x$ and $\Re(\cdot)$ denotes the real part.
A matrix $M$ is negative definite if $-M$ is positive definite.
Likewise, a matrix $M$ is positive semi-definite if $\Re(x^H M x) \geq 0$ for any non-zero $x \in \C[n]$. 
A matrix $M$ is negative semi-definite if $-M$ is positive semi-definite.
Given a vector $x \in \C[n]$, 
we define the $\ell_2$ norm $\norm{x} \doteq \sqrt{x^H x}$.
The $\ell_2$ norm $\norm{\cdot}$ also induces a matrix norm.
Given a matrix $M \in \C[m \times n]$,
the induced matrix norm is defined as
$\norm{M} \doteq \sup_{x \in \C[n], x \neq 0} \frac{\norm{Mx}}{\norm{x}}$.
We now restrict ourselves to real vectors and matrices.
A real symmetric positive definite matrix $D \in \R[n]$ induces a vector norm $\Dnorm{\cdot}$, where $\Dnorm{x} \doteq \sqrt{x^\top D x}$ for $x \in \R[n]$.
We overload $\norm{\cdot}_D$ to also denote the induced matrix norm.

We consider a Markov Reward Process (MRP\footnote{In policy evaluation, the policy $\pi$ is fixed and induces an MRP from a Markov Decision Process (MDP).},~\citet{bellman1957markovian,puterman2014markov}) with a state space $\fS \subseteq \R[K]$ for some $K \in \mathbb{N}$.
The MRP employs an initial distribution $p_0: \fB(\fS) \to [0, 1]$, where $\fB$ denotes the Borel algebra.
The dynamics of the MRP are characterized by a transition kernel $p: \fB(\fS) \times \fS  \to [0, 1]$\footnote{We define the kernel this way to be compatible with the conventional notation $p(\cdot | s)$}.
The MRP also adopts a measurable and bounded reward function $r: \fS \to \R$. 
At time step $t$, suppose the agent is at state $S_t$.
It transitions to the next state $S_{t+1} \sim p(\cdot | S_t)$ following the transition kernel.
At the same time, the environment emits a reward $R_{t+1} \doteq r(S_t)$ to the agent for transitioning out of $S_t$.
Note that $(p, p_0)$ defines a Markov chain $\qty{S_t}$.
We further define a $p$-induced operator $P_\pi: \fF \to \fF$, 
where $\fF \doteq \qty{f: \fS \to \R[n] \mid f \text{ is measurable and integrable}}$, as
$(P_\pi f)(s) \doteq \int_\fS f(s') p (\dd{s'} \mid s)$.

The state-value function $v_\pi: \fS \to \R$ is defined as
$v_\pi(s) \doteq \E\qty[\sum_{t=0}^\infty \gamma^t R_{t+1} | S_0 = s]$, 
where $\gamma \in [0,1)$ is a discount factor.
The state-value function maps each state to the expected cumulative reward the agent gets starting from that state.
We consider a linear function approximation of $v_\pi$ with a feature mapping $x: \fS \to \R[d]$ and a weight vector $w \in \R[d]$.
The function $x$ maps each state $s \in \fS$ to a $d$-dimensional real vector $x(s)$.
One simply takes $x(s)^\top w$ to compute the approximated state value for $s$.
The goal is thus to adjust $w$ such that $x(s)^\top w \approx v_\pi(s)$ for all $s\in\fS$.
Linear TD updates the weight $w$ iteratively as
\begin{align}
  \label{eqn: linear td update}
  \tag{Linear TD}
  w_{t+1} = w_t + 
  \alpha_t\qty(R_{t+1} + \gamma x(S_{t+1})^\top w_t - x(S_t)^\top w_t) x(S_t).
\end{align}
where $\alpha_t$ is the learning rate and we recall that $\qty{S_t}$ is a Markov chain evolving as $S_{t+1} \sim p(\cdot |S_t)$ and $S_0$ is sampled from an arbitrary initial distribution.
The following assumptions are commonly made in analyzing linear TD~\citep{tsitsiklis1997analysis}.
\begin{xassumption}
  \label{assumption: learning rate}
  The learning rates $\qty{\alpha_t}$ is a decreasing sequence of positive real numbers such that $\sum_{t=0}^\infty \alpha_t = \infty$, $\sum_{t=0}^\infty \alpha_t^2 < \infty$, and $\lim_{t\to\infty} \qty(\frac{1}{\alpha_{t+1}} - \frac{1}{\alpha_t}) < \infty$.
\end{xassumption}

\begin{xassumption}
  \label{assumption: ergodicity}
  The Markov chain $\qty{S_t}$ admits a well-defined and unique stationary distribution $\mu: \fB(\fS) \to [0, 1]$, such that $\mu(Z) = \int_{\fS} p(Z \mid s)\mu(\dd{s})$ for all $Z \in \fB(\fS)$ and $\mu(U) > 0$ for all non-empty open sets $U \subseteq \fS$.
\end{xassumption}
The sequence $\qty{\frac{1}{(t+1)^k}}$ where $k \in (0.5, 1]$ satisfies Assumption~\ref{assumption: learning rate} as a valid example.
In light of Assumption~\ref{assumption: ergodicity}, we define an inner product $\indot{\cdot}{\cdot}_\mu$ induced by $\mu$ as
$\indot{f}{g}_\mu \doteq \int_\fS \indot{f(s)}{g(s)} \mu(\dd{s})$,
where $f:\fS \to \R[n]$ and $g:\fS \to \R[n]$ are measurable functions.
The inner product $\indot{\cdot}{\cdot}_\mu$ further induces a semi-norm $\norm{f}_\mu^2 \doteq \indot{f}{f}_\mu$.
\begin{xassumption}
  \label{assumption: feature function bounded}
  The feature function and reward function are bounded, i.e., \\$\sup_{s\in\fS} \norm{x(s)} < \infty, \sup_{s\in\fS} \abs{r(s)} < \infty$.
\end{xassumption}
We use $L_p(\fS, \mu)$ to denote the set of functions defined on $\fS$ that are $p$-integrable with respect to $\mu$.
A direct consequence of Assumption~\ref{assumption: feature function bounded} is that $x \in L_p(\fS, \mu)$ for $p \in [1,\infty)$ because
$\qty(\int_\fS \norm{x(s)}^p \mu(\dd{s}))^{1/p} 
\le \qty(\qty(\sup_{s\in\fS} \norm{x(s)})^p \int_\fS \mu(\dd{s}))^{1/p}
= \sup_{s\in\fS} \norm{x(s)} < \infty$.
\begin{xassumption}
  \label{assumption: transition kernel map zero to zero}
  Given any $f: \fS \to \R[n]$ and $f = 0$ a.e. with respect to $\mu$, it holds that $P_\pi f = 0$ a.e. with respect to $\mu$ as well.
\end{xassumption}
In the rest of the paper, we use ``a.e.'' to denote ``a.e. with respect to $\mu$'' for simplifying notations. 
Notably, Assumptions~\ref{assumption: ergodicity} -~\ref{assumption: transition kernel map zero to zero} trivially hold when $\fS$ is finite and $\qty{S_t}$ is irreducible.

We can represent the feature mapping $x$ as 
$x(s) = \mqty[x_1(s) & x_2(s) & \cdots & x_d(s)]^\top$, where each $x_i: \fS \to \R$ is a basis function 
for $i = 1,2,\dots,d$.
Since it is defined on a general state space, we define the linear dependence/independence of the features in the a.e. sense.
We say the set of basis functions $\qty{x_1, x_2, \dots, x_d}$ are linearly dependent a.e. if there exists $c \in \R[d]\setminus\qty{0}$, such that $c_1 x_1 + c_2 x_2 + \dots + c_d x_d = 0$ a.e.
Likewise, we say the set of basis functions $\qty{x_1, x_2, \dots, x_d}$ are linearly independent a.e. if $c_1 x_1 + c_2 x_2 + \dots + c_d x_d = 0$ a.e. holds if and only if $c = 0$.
In the rest of the text, we omit a.e. from linear dependence/independence whenever it does not confuse.

Intuitively, when $\qty{S_t}$ reaches the stationary distribution and the learning rate $\alpha$ is ``small'', linear TD behaves like a deterministic algorithm as
$w_{t+1} = w_t + \alpha_t\qty(Aw_t + b)$,
where 
\begin{align}
  A &\doteq \textstyle\E_{S \sim \mu, S' \sim p(\cdot \mid S)}
  \qty[x(S) \qty(\gamma x(S')^\top - x(S)^\top)]\\
  &\textstyle= \int_\fS 
  x(s)\qty(\gamma (P_\pi x)(s)^\top - x(s)^\top) \mu(\dd{s})\in \R[d \times d] \label{eq: A definition} \\
  b & \doteq \E_{S \sim \mu}\qty[x(S)r(S)]
  = \textstyle\int_{\fS} x(s)r(s)\mu(\dd{s}) \label{eq: b definition}
  \in \R[d].
\end{align}
We can then further relate the stochastic and discrete iterative updates~\eqref{eqn: linear td update} with deterministic and continuous trajectories of the following ordinary differential equation (ODE)
\begin{align}
  \label{eq linear td ode}
  \textstyle \dv{w(t)}{t} = Aw(t) + b,
\end{align}
which is known as the ODE method in stochastic approximation \citep{benveniste1990MP,kushner2003stochastic,borkar2009stochastic,borkar2025ode,liu2025ode}.

In the finite case, it is also well known that $P_\pi$ is nonexpansive and $A$ is negative semi-definite~\citep{tsitsiklis1997analysis} (if $\qty{x_i}$ are linearly independent, then $A$ would be negative definite).
These are also true in the continuous state space we consider.
\begin{lemma}
  \label{lemma: pv norm}
  Let Assumption~\ref{assumption: ergodicity} hold. For any $v \in L_2(\fS, \mu)$, it holds that $\norm{P_\pi v}_\mu \le \norm{v}_\mu$.
\end{lemma}
See Appendix~\ref{proof: pv norm} for the proof.

\begin{lemma}
  \label{lemma: A NSD}
  Let Assumptions~\ref{assumption: ergodicity} \&~\ref{assumption: feature function bounded} hold. Then
  $A$ is negative semi-definite.
\end{lemma}
See Appendix~\ref{proof: A NSD} for the proof.

\section{TD Fixed Points}
\label{section: TD fixed points}
As previously discussed, linear TD is approximately a stochastic discretization of ODE~\eqref{eq linear td ode}.
The fixed points of linear TD, commonly referred to as TD fixed points~\citep{tsitsiklis1997analysis,sutton2018reinforcement}, are linked to the equilibria of this ODE.
We hence consider the linear system 
\begin{align}
  \label{eq: linear equation}
  A w + b = 0,
\end{align}
where $A$ is defined in~\eqref{eq: A definition} and $b$ is defined in~\eqref{eq: b definition}.
With linearly independent features,
\eqref{eq: linear equation} adopts a unique solution --- the unique TD fixed point.
Without assuming linear independence,
matrix $A$ is merely negative semi-definite (Lemma~\ref{lemma: A NSD}),
so~\eqref{eq: linear equation} can potentially adopt infinitely many solutions.
In light of this,
we refer to all solutions to the linear system as TD fixed points.
Namely,
we define
$\fW_* \doteq \qty{w_* \middle| Aw_* + b = 0}$
and refer to $\fW_*$ as \emph{the set of TD fixed points}.
A few questions arise naturally from this definition.
\begin{enumerate}[(Q1)]
  \item Is $\fW_*$ always non-empty?
  \label{question nonempty}
  \item If $\fW_*$ contains multiple weights, do those weights give the same value estimate?
  \label{question same value}
  \label{question mspbe}
  \item Do the iterates $\qty{w_t}$ generated by~\eqref{eqn: linear td update} converge to $\fW_*$?
  \label{question convergence}
\end{enumerate}
We shall give affirmative answers to all the questions above in the rest of the paper.
We use 
\begin{align}
  \label{eq: v_w definition}
  \textstyle v_w(s) \doteq x(s)^\top w
\end{align} to denote the value estimate for a state $s$ given a weight $w$.
Below, we answer (Q\ref{question nonempty}) and (Q\ref{question same value}) affirmatively as a warm-up. 
The affirmative answer to (Q\ref{question convergence}) is much more involved and is deferred to the next section.
\begin{lemma}
  \label{lemma: inner integral positive definite}
  Let Assumptions~\ref{assumption: ergodicity} \&~\ref{assumption: feature function bounded} hold.
  Then, $w^\top A w = 0 \iff v_w = 0$ a.e.
\end{lemma}
The proof is in Appendix~\ref{proof: inner integral positive definite}.
\begin{lemma}
  \label{lemma: W nonempty}
  Let Assumptions~\ref{assumption: ergodicity},~\ref{assumption: feature function bounded}, \&~\ref{assumption: transition kernel map zero to zero} hold. Then $\fW_*$ is non-empty.
\end{lemma}
The proof is in Appendix~\ref{proof: W nonempty}.
This answers (Q\ref{question nonempty}) affirmatively. 
\begin{lemma}
  \label{lemma: v = v' a.e. iff w' in W}
  Let Assumptions~\ref{assumption: ergodicity} \&~\ref{assumption: feature function bounded} hold.
  Given any $w\in\fW_*$ and any $w' \in \R[d]$,
  it holds that $v_w = v_{w'} \,\, {a.e.} \iff w' \in \fW_*$.
\end{lemma}
The proof is in Appendix~\ref{proof: v = v' a.e. iff w' in W}.
This answers (Q\ref{question same value}) affirmatively. 

\section{ODE Solutions}
\label{section: ode analysis}
In this section, we study the mean ODE~\eqref{eq linear td ode} related to linear TD and establish its convergence.
We use $w(t; w_0)$ to denote the solution to \eqref{eq linear td ode} with the initial condition $w(0; w_0) = w_0$.
Recall that matrix $A$ would be negative definite if the features are linearly independent.
Then, it would follow from standard dynamical system results~\citep{khalil2002nonlinear} that
$\lim_{t\to\infty} w(t; w_0) = -A^{-1}b$.
In other words,
regardless of the initial condition $w_0$,
a solution always converges to the globally asymptotically stable equilibrium $-A^{-1}b$.
Without assuming linear independence,
matrix $A$ is merely negative semi-definite (Lemma~\ref{lemma: A NSD}).
It is then impractical to expect all solutions to converge to the same point.
However,
can we still expect each $w(t; w_0)$ to converge to a $w_0$-dependent limit?
The answer is affirmative.
To proceed, we first perform a standard change of variable.
For any $w_0 \in \R[d]$ and any $w_* \in \fW_*$,
we have
  $\dv{\qty(w(t; w_0) - w_*)}{t} 
    = \dv{w(t; w_0)}{t}
    = Aw(t;w_0) + b - (Aw_* + b)
    = A(w(t; w_0) - w_*)$.
This indicates that $w(t; w_0) - w_*$ is a solution to the shifted ODE
  \begin{align}
    \label{eq shifted ode}
    \textstyle\dv{z(t)}{t} = A z(t).
  \end{align}
starting from $w_0 - w_*$.
Therefore,
to study the original ODE~\eqref{eq linear td ode},
it is sufficient to study the shifted ODE~\eqref{eq shifted ode}.
We use $z(t; z_0)$ to denote a solution to~\eqref{eq shifted ode} with the initial condition $z(0; z_0) = z_0$.
We then have
\begin{align}
  \label{eq z eq w}
  z(t;w_0 - w_*) = w(t;w_0) - w_*
\end{align}
When it does not confuse, we write $z(t; z_0)$ as $z(t)$ for simplicity.
Analogous to the definition of $\fW_*$, we define
$\fZ_* \doteq \qty{z \Big| Az = 0}$.

\begin{corollary}
  \label{corollary: vz equals 0 a.e. iff z in Z}
  Let Assumptions~\ref{assumption: ergodicity} \&~\ref{assumption: feature function bounded} hold.
  Then $v_z = 0 \,\, a.e. \iff z \in \fZ_*$.
\end{corollary}
\begin{proof}
We recall that Lemma~\ref{lemma: v = v' a.e. iff w' in W} holds for any reward function $r$. 
By letting $r = 0$, we obtain $b = 0$ and $\fW_* = \fZ_*$.
Therefore, given any $z_* \in \fZ_*$ and any $z \in \R[d]$, it holds that $v_{z_*} = v_z\quad\text{a.e.} \iff z\in\fZ_*$.
Furthermore, since $z_* \in \fZ_*$, it holds that $Az_* = 0$, which implies $z_*^\top A z_* = 0$.
Hence, by Lemma~\ref{lemma: inner integral positive definite}, it holds that $v_{z_*} = 0$ a.e.
As a result, we have $v_z = 0 \,\, a.e. \iff z \in \fZ_*$, which completes the proof.
\end{proof}
\subsection{Value Convergence}
Here, we prove the value estimate of the mean ODE~\eqref{eq linear td ode} converges for almost all states.
\begin{theorem}
  \label{thm: ode value convergence}
  Let Assumptions~\ref{assumption: ergodicity},~\ref{assumption: feature function bounded}, \&~\ref{assumption: transition kernel map zero to zero} hold.
  For any $w_0 \in \R[d]$ and any $w_* \in \fW_*$,
  there exists $\fS^+ \subseteq \fS$ with $\mu(\fS^+) = 1$, such that $\forall s \in \fS^+$,
  $\lim_{t\to\infty}  v_{w(t;w_0)}(s) = v_{w_*}(s)$.
\end{theorem}
\begin{proof}
Fix an arbitrary $z_* \in \fZ_*$.
We define $U(z) \doteq \frac{1}{2}\norm{z - z_*}^2$.
Then, for any $z_0 \in \R[d]$,
we have
\begin{align}
  \label{eq dz le 0}
  \textstyle \dv{U(z(t))}{t} =& 
  \textstyle (z(t) - z_*)^\top \dv{z(t)}{t}= \textstyle (z(t) - z_*)^\top A z(t)\\
  =& \textstyle (z(t) - z_*)^\top A(z(t) - z_*)
  \leq \textstyle 0 \explain{Lemma~\ref{lemma: A NSD}}.
\end{align}
We now claim 
\begin{align}
  \textstyle
  \label{eq zeros of shifted ode}
  \dv{U(z(t))}{t} = 0 \iff z(t) \in \fZ_*.
\end{align}
To see $\textstyle \dv{U(z(t))}{t} = 0 \impliedby z(t) \in \fZ_*$, we note $Az(t) = 0$ by definition when $z(t) \in \fZ_*$.
Hence, $\textstyle \dv{U(z(t))}{t} = (z(t) - z_*)^\top A z(t) = 0$.
To see $ \dv{U(z(t))}{t} = 0 \implies z(t) \in \fZ_*$, we note that $v_{z(t) - z_*} = 0$ a.e. when $\dv{U(z(t))}{t} = 0$ according to Lemma~\ref{lemma: inner integral positive definite}.
That implies $z(t) - z_* \in \fZ_*$ by Corollary~\ref{corollary: vz equals 0 a.e. iff z in Z}, which further implies $z(t) \in \fZ_*$ by the definition of $\fZ_*$.

The result in~\eqref{eq zeros of shifted ode} suggests that $U$ is almost a Lyapunov function except that $\dv{U(z(t))}{t}$ has multiple zeros.
The standard Lyapunov stability theorem (e.g., Theorem 4.1 of~\citet{khalil2002nonlinear}) thus does not apply.
\begin{xtheorem}
\label{thm: lasalle theorem}
  (LaSalle's theorem, Theorem 4.4 of~\citet{khalil2002nonlinear})
  Let $\Omega \subset \R[d]$ be a compact set that is positively invariant\footnote{A set $Z$ is a positively invariant set of~\eqref{eq shifted ode} if $z(0) \in Z \implies \forall t \in [0, \infty), z(t) \in Z$. Here, $z(t)$ denotes a solution to~\eqref{eq shifted ode} on $[0, \infty)$.} with respect to~\eqref{eq shifted ode}. 
  Let $U: \R[d] \to \R$ be a continuously differentiable function such that $\dv{U(z(t))}{t} \le 0$ whenever $ z(t) \in \Omega$.
  Let $E$ be the set of all points in $\Omega$ satisfying $\dv{U(z(t))}{t} = 0$ whenever $z(t) \in E$. 
  Let $M$ be the largest invariant set\footnote{A set $Z$ is an invariant set of~\eqref{eq shifted ode} if $z(0) \in Z \implies \forall t \in (-\infty, \infty), z(t) \in Z$. Here, $z(t)$ denotes a solution to~\eqref{eq shifted ode} on $(-\infty, \infty)$.} in $E$. 
  Then every solution $z(t)$ with $z_0 \in \Omega$ satisfies\footnote{The distance between a point $z$ and a set $\Omega$ is defined as $d(z, \Omega) \doteq \inf_{z' \in \Omega} \norm{z - z'}$.}
  $\lim_{t\to\infty} d(z(t), M) = 0$.
\end{xtheorem}
Given any $z_0\in\R[d]$, 
to apply LaSalle's theorem,
we define $\Omega \doteq \qty{z \mid \norm{z - z_*} \le \norm{z_0 - z_*}}$.
This set $\Omega$ is compact because it defines a closed ball centered on $z_*$ in $\R[d]$, thus is closed and bounded.
We first show that $\Omega$ is positively invariant.
For any $z_0' \in \Omega$,
it holds that for any $t \geq 0$,
$\norm{z(t; z_0') - z_*}^2 \leq \norm{z(0; z_0') - z_*}^2 = \norm{z_0' - z_*}^2 \leq \norm{z_0 - z_*}^2$,
where the first inequality holds because $\dv{\norm{z(t; z_0') - z_*}^2}{t} \leq 0$ by~\eqref{eq dz le 0} and the second inequality holds because $z_0' \in \Omega$.
Then, we conclude that $z(t; z_0') \in \Omega$ for any $t \geq 0$,
implying that $\Omega$ is positively invariant.
We use our previously defined $U$ as the $U$ for Theorem~\ref{thm: lasalle theorem}. Then $\dv{U(z(t))}{t} \leq 0$ holds by~\eqref{eq dz le 0}.

In light of~\eqref{eq zeros of shifted ode}, set $E$ defined in Theorem~\ref{thm: lasalle theorem} is then
$E = \fZ_* \cap \Omega$. 
We now show that $E$ itself is an invariant set, so that the set $M$ defined in Theorem~\ref{thm: lasalle theorem} is just $E$.
Let $z(t; z_0)$ be a solution to~\eqref{eq shifted ode} in $(-\infty, \infty)$ with $z_0 \in E$.
Define $\fS^+_{z_0} \doteq \qty{s \mid v_{z_0}(s) = 0, s \in \fS}$.
Since $z_0 \in \fZ_*$, it holds that $\mu(\fS^+_{z_0}) = 1$ by Corollary~\ref{corollary: vz equals 0 a.e. iff z in Z}.
Then, for any $t \in (-\infty, \infty)$ and $s' \in \fS^+_{z_0}$,
we have 
\begin{align}
  \label{eq zt integral form}
  z(t) =& \textstyle z_0 + \int_0^t Az(\tau) \dd{\tau}, \\
  x(s')^\top z(t)
  =& \textstyle x(s')^\top z_0
  +  \int_0^t \int_\fS x(s')^\top  x(s)\qty(\gamma (P_\pi x)(s)^\top - x(s)^\top)  \mu(\dd{s})z(\tau) \dd{\tau}\\ 
  =& \textstyle \int^t_0 \int_\fS x(s')^\top  x(s)\qty(\gamma (P_\pi x)(s)^\top - x(s)^\top) \mu(\dd{s}) z(\tau)\dd{\tau} \explain{$s'\in\fS^+_{z_0}$}.
\end{align}
Then, recall the definition of $v_w$ in~\eqref{eq: v_w definition}, we have
\begin{align}
  v_{z(t)}(s')
  =&\textstyle\int_0^t \int_\fS 
    v_{x(s')}(s)\qty(\gamma(P_\pi v_{z(\tau)})(s) - v_{z(\tau)}(s))  
    \mu(\dd{s})\dd{\tau}\\
  =&\textstyle\int_0^t
    \indot{v_{x(s')}}{\gamma P_\pi v_{z(\tau)} - v_{z(\tau)}}_\mu \dd{\tau}
  =\textstyle\int_t^0
    \indot{v_{x(s')}}{v_{z(\tau)} - \gamma P_\pi v_{z(\tau)}}_\mu \dd{\tau}.
\end{align}
Therefore, when $t \ge 0$, we get
\begin{align}
   v_{z(t)}(s')^2
   =&\textstyle\qty(
     \int_0^t
     \indot{v_{x(s')}}{\gamma P_\pi v_{z(\tau)} - v_{z(\tau)}}_\mu
     \dd{\tau})^2\\
    \le&\textstyle t\int_0^t
        \indot{v_{x(s')}}{\gamma P_\pi v_{z(\tau)} - v_{z(\tau)}}_\mu^2
        \dd{\tau} \explain{Cauchy-Schwarz inequality}\\
    \le&\textstyle t\norm{v_{x(s')}}^2_\mu \int_0^t
        \norm{\gamma P_\pi v_{z(\tau)} - v_{z(\tau)}}_\mu^2
        \dd{\tau} \explain{Cauchy–Schwarz inequality}.
\end{align}
It then holds that
\begin{align}
  \textstyle \int_{\fS} v_{z(t)}(s')^2\mu(\dd{s'})
  \le& \textstyle t\int_{\fS}\norm{v_{x(s')}}^2_\mu \mu(\dd{s'})
      \int_0^t
      \norm{\gamma P_\pi v_{z(\tau)} - v_{z(\tau)}}_\mu^2
      \dd{\tau}\\
  \le&\textstyle t C
      \int_0^t
      \norm{\gamma P_\pi v_{z(\tau)} - v_{z(\tau)}}_\mu^2
      \dd{\tau}
      \explain{$C \doteq \sup_{s' \in \fS} \norm{v_{x(s')}}_\mu^2 < \infty$}\\
  \le&\textstyle t C
      \int_0^t \qty(
      \gamma^2\norm{P_\pi v_{z(\tau)}}^2_\mu
      + 2\gamma \norm{P_\pi v_{z(\tau)}}_\mu\norm{v_{z(\tau)}}_\mu
      + \norm{v_{z(\tau)}}_\mu^2)
      \dd{\tau}\\
  \le& \textstyle t C (\gamma + 1)^2 \int_0^t \norm{v_{z(\tau)}}_\mu^2 \dd{\tau} \explain{Lemma~\ref{lemma: pv norm}}.
\end{align}
We thus have $\norm{v_{z(t)}}_\mu^2 \le t C (\gamma + 1)^2 \int_0^t \norm{v_{z(\tau)}}_\mu^2 \dd{\tau}$ for $t \ge 0$.
By Gronwall's inequality (see, e.g., Theorem~A.1 of \citet{liu2025ode}),  we get
$\norm{v_{z(t)}}_\mu^2 \le 0 \exp(t^2 C (\gamma +1 )^2) = 0$
for $t \ge 0$.
Similarly, when $t \le 0$, we have
\begin{align}
  v_{z(t)}(s')^2
  =&\qty(
     \int_t^0
     \indot{v_{x(s')}}{v_{z(\tau)} - \gamma P_\pi v_{z(\tau)}}_\mu
     \dd{\tau})^2\\
  \le&-t\int_t^0
        \indot{v_{x(s')}}{v_{z(\tau)} - \gamma P_\pi v_{z(\tau)}}_\mu^2
        \dd{\tau}
        \explain{Cauchy-Schwarz inequality}\\
  \le&-t\norm{v_{x(s')}}^2_\mu\int_t^0
        \norm{v_{z(\tau)} - \gamma P_\pi v_{z(\tau)}}_\mu^2
        \dd{\tau}.
        \explain{Cauchy-Schwarz inequality}
\end{align}
Following the same steps as above, we arrive at $\norm{v_{z(t)}}_\mu^2 \le -t C (\gamma + 1)^2 \int_t^0 \norm{v_{z(\tau)}}_\mu^2 \dd{\tau}$ for $t \le 0$.
By a reverse time version of Gronwall's inequality (see, e.g., Theorem A.2 of \citet{liu2025ode}),  it holds that
$\norm{v_{z(t)}}_\mu^2 \le 0 \exp(t^2 C (\gamma +1 )^2) = 0$
for $t \le 0$.
The fact $\norm{v_{z(t)}}_\mu^2 = 0$ for $t \in (-\infty, \infty)$ implies that $v_{z(t)} = 0$ a.e. for $t\in(-\infty, \infty)$.
According to Corollary~\ref{corollary: vz equals 0 a.e. iff z in Z}, we have $z(t) \in \fZ_*$ for $t \in (-\infty, \infty)$.
This means $Az(\tau) = 0$ in~\eqref{eq zt integral form},
implying that $z(t) = z_0 \in E$ for all $t \in (-\infty, \infty)$.
We have now proved that $E$ is an invariant set of~\eqref{eq shifted ode}.
Theorem~\ref{thm: lasalle theorem} then implies that for any $z_0$,
\begin{align}
  \textstyle
  \label{eq: z approaches solutions set}
  \lim_{t\to\infty} d(z(t; z_0), E) = 0.
\end{align}

We now convert the convergence of $z(t; z_0)$ back to $w(t; w_0)$.
To this end, fix any $z_* \in E$. 
Define $\fS^+_{z_*} \doteq \qty{s \mid x(s)^\top z_* = 0, s\in\fS}$.
Since $z_* \in E$, it holds that $\mu(\fS^+_{z_*}) = 1$ by Corollary~\ref{corollary: vz equals 0 a.e. iff z in Z}.
For any $s \in \fS^+_{z_*}$, it holds that
$\inf_{z\in E} \abs{\indot{x(s)}{z(t; z_0) - z}}
  =\inf_{z\in E} \abs{\indot{x(s)}{z(t; z_0)} - \indot{x(s)}{z}}\\
  \ge\abs{\indot{x(s)}{z(t; z_0)}} - \abs{\indot{x(s)}{z_*}}$.
Therefore, we have
\begin{align}
  \textstyle
  \lim_{t\to\infty} \abs{\indot{x(s)}{z(t; z_0)}}
  &\le \textstyle \lim_{t\to\infty}\inf_{z\in E} \abs{\indot{x(s)}{z(t; z_0) - z}} + \abs{\indot{x(s)}{z_*}}\\
  &\le \textstyle \norm{x(s)} \lim_{t\to\infty} \inf_{z\in E} \norm{z(t; z_0) - z} + \abs{\indot{x(s)}{z_*}} \\
  &= \norm{x(s)} \lim_{t\to\infty}d(z(t; z_0), E) + \abs{\indot{x(s)}{z_*}}
  = \abs{\indot{x(s)}{z_*}} \explain{By~\eqref{eq: z approaches solutions set}}
  = 0.
\end{align}
It then follows immediately that 
\begin{align}
  \textstyle
  \label{eq z value converge}
  \lim_{t\to\infty} \indot{x(s)}{z(t; z_0)} = 0.
\end{align}
Since~\eqref{eq z value converge} holds for any $z_0$ and corresponding trajectory $z(t; z_0)$,
it also holds for the trajectory $w(t;w_0) - w_*$ that starts from $w_0 - w_*$, i.e.,
\begin{align}
  \textstyle \lim_{t\to\infty} x(s)^\top (w(t; w_0) - w_*) &= 0\\
  \textstyle \lim_{t\to\infty} v_{w(t;w_0)}(s) &= v_{w_*}(s),
\end{align}
which completes the proof.
\end{proof}

\subsection{Weight Convergence}
The value convergence in Theorem~\ref{thm: ode value convergence} immediately implies that any solution $w(t; w_0)$ would eventually converge to the set $\fW_*$ as time progresses.
But is it possible that $w(t; w_0)$ keeps oscillating within $\fW_*$ or in neighbors of $\fW_*$ without ever converging to any single point?
In this section,
we rule out this possibility and prove that any solution will always converge to some fixed point, formalized in the next theorem.

\begin{theorem}
  \label{thm: ode weight convergence}
  Let Assumptions~\ref{assumption: ergodicity}, \ref{assumption: feature function bounded}, \&~\ref{assumption: transition kernel map zero to zero} hold.
  For any $w_0 \in \R[d]$,
  there exists a constant $w_\infty(w_0) \in \fW_*$, such that
  $\lim_{t\to\infty} w(t; w_0) = w_\infty(w_0)$.
\end{theorem}
We shall perform a finer analysis of the mean ODE~\eqref{eq shifted ode} and propose several helper lemmas to prove this theorem.
Firstly, it is well-known that $z(t; z_0)$ has a closed-form solution~\citep{khalil2002nonlinear} as
\begin{align}
  \textstyle
  \label{eqn: closed form solution}
  z(t;z_0) = \exp(At)z_0.
\end{align}
The standard approach to work with matrix exponential\footnote{For a square matrix $X$, its exponential is $\exp(X) \doteq \sum_{n=0}^\infty \frac{1}{n!} X^n$.} is to consider the Jordan normal form~\citep{horn2012matrix}.
Namely,
we decompose $A = PJP^{-1}$.
Here, $P$ is the invertible matrix in Jordan decomposition and $J$ is the Jordan matrix of $A$.
We use $\lambda_1, \lambda_2, \dots, \lambda_k$
to denote the $k$ distinct eigenvalues of $A$.
We use $m_1, \dots, m_k$ to denote the algebraic multiplicity of each eigenvalue.
Likewise, we use $g_1, \dots, g_k$ to denote the geometric multiplicity of each eigenvalue.
It always holds that $1 \le g_i \le m_i$ for $i = 1, 2, \dots, k$.
Each distinct eigenvalue $\lambda_i$ has exactly $g_i$ corresponding Jordan blocks.
The dimensions of each Jordan block are inconsequential for our analysis.
Therefore, to simplify our notation, we use $\rho_{i,j}$ to denote the dimension of the $j$-th Jordan block of $\lambda_i$.
Notably, $m_i = \sum_{j=1}^{g_i} \rho_{i,j}$.
Then, the Jordan matrix can be expressed as
$J =  \bigoplus_{i=1}^{k} \bigoplus_{j=1}^{g_i} B_{i,j}$,
where $B_{i,j} \in \C[\rho_{i,j} \times \rho_{i,j}]$ is the $j$-th Jordan block corresponding to the eigenvalue $\lambda_i$ and $\bigoplus$ denotes the direct matrix sum\footnote{Given block matrices $A$ and $B$, $A \bigoplus B = \mqty[A & \\ & B]$.}.
The matrix exponential can then be computed as
\begin{align}
  \label{eqn: exp(At) Jordan normal form}
  \exp(At)
  &= \textstyle \sum_{n=0}^{\infty} \frac{1}{n!} (P J P^{-1})^n t^n
  = P \exp(Jt) P^{-1} \\
  &= \textstyle P \exp(\bigoplus_{i=1}^{k} \bigoplus_{j=1}^{g_i} B_{i,j} t) P^{-1} 
  = P \left[\bigoplus_{i=1}^{k} \bigoplus_{j=1}^{g_i} \exp(B_{i,j} t)\right] P^{-1}.
\end{align}
As an example of the Jordan blocks, suppose $\rho_{1,1} = 3$, then
  $B_{1,1} = \mqty[\lambda_1 & 1 & 0\\
                  0 & \lambda_1 & 1\\
                  0 & 0 & \lambda_1]$.
We may also express each Jordan block as $B_{i,j} = \lambda_i I_{\rho_{i,j}} + N_{i,j}$,
where $N_{i,j}$ is a nilpotent matrix\footnote{A nilpotent matrix~(Section 0.9.13 of \citet{horn2012matrix}) is a square matrix $N$, such that $N^k$ = 0 for some positive integer $k$. The smallest $k$ is called the index of $N$.}.
Continuing with our example,
  $N_{1,1} = \mqty[0 & 1 & 0\\
                  0 & 0 & 1\\
                  0 & 0 & 0]$.
We denote the index of $N_{i,j}$ as $y_{i,j}$.
Using the property of a nilpotent matrix and matrix exponential,
we can compute $\exp(B_{i,j} t)$ as
\begin{align}
  \label{eq exp bij t}
  \textstyle \exp(B_{i,j} t)
  =& \textstyle \exp(\lambda_i I_{\rho_{i,j}} t + N_{i,j} t)
  = \textstyle \exp(\lambda_i t) \exp(N_{i,j} t)\\
  =& \textstyle \exp(\lambda_i t) \sum_{n=0}^\infty \frac{1}{n!} t^n N_{i,j}^n \explain{power series}\\
  =& \textstyle \exp(\lambda_i t) \sum_{n=0}^{y_{i,j}-1} \frac{1}{n!} t^n N_{i,j}^n \explain{$N_{i,j}^{n} = 0, \forall n \ge y_{i,j}$}\\
  \label{eqn: exp^Bt}
  =& \textstyle \exp(\Re(\lambda_i)t) \exp(\mi \Im(\lambda_i)t) \sum_{n=0}^{y_{i,j}-1} \frac{1}{n!} t^n N_{i,j}^n \\
  =& \textstyle \exp(\Re(\lambda_i)t)\left(\cos(\Im(\lambda_i)t) + \mi \sin(\Im(\lambda_i))\right) \sum_{n=0}^{y_{i,j}-1} \frac{1}{n!} t^n N_{i,j}^n \explain{Euler's formula}.
\end{align}
Here, $\Im(\cdot)$ denotes the imaginary part.
The fact that matrix $A$ is negative semi-definite (Lemma~\ref{lemma: A NSD}) implies $\Re(\lambda_i) \leq 0$.\footnote{Let $A$ be a negative semi-definite matrix and $u$ be an eigenvector of A having eigenvalue $\lambda$, it holds that $\Re(u^H A u) = \Re(\lambda\norm{u}^2) = \Re(\lambda)\norm{u}^2 \le 0$. Since $\norm{u}^2 > 0$, we have $\Re(\lambda) \le 0$.}
We now branch into cases for $\Re(\lambda) < 0$ and $\Re(\lambda) = 0$.
\begin{lemma}
  \label{lemma: convergence re(lambda) negative}
  Let Assumptions~\ref{assumption: ergodicity} \&~\ref{assumption: feature function bounded} hold.
  If $\Re(\lambda_i) < 0$, then it holds that
  $\lim_{t \to \infty} \\\exp(B_{i,j} t) = 0 \qq{$\forall j \in \qty{1, \dots, g_i}$.}$
\end{lemma}
The proof is in Appendix~\ref{proof: convergence re(lambda) negative}.
This is intuitive because all terms other than $\exp(\Re(\lambda_i)t)$ are at most polynomial and are dominated by the exponential.
To analyze the case of $\Re(\lambda_i) = 0$, we make the following two observations.
\begin{lemma}
  \label{lemma: exp^Bt bounded}
  Let Assumptions~\ref{assumption: ergodicity} \&~\ref{assumption: feature function bounded} hold.
  $\forall i, j$,
  it holds that $\sup_{t\in [0, \infty)} \norm{\exp(B_{i,j} t)} < \infty$.
\end{lemma}
The proof is in Appendix~\ref{proof of e^Bt bounded}.
Intuitively, the boundedness holds because $w(t; w_0)$ is bounded for any $w_0\in\R[d]$ (Lemma~\ref{lemma: w(t) bounded}).
Furthermore, it follows from the value convergence that
\begin{corollary}
\label{corollary: dz/dt diminish}
Let Assumpitons~\ref{assumption: ergodicity},~\ref{assumption: feature function bounded}, \&~\ref{assumption: transition kernel map zero to zero} hold. Then  $\forall z_0 \in \R[d]$,
$\lim_{t\to\infty} \dv{z(t; z_0)}{t} = 0$.
\end{corollary}
The proof is in Appendix~\ref{proof: dz/dt diminish}.
We are now ready to discuss the case $\Re(\lambda_i) = 0$.
\begin{lemma}
  \label{lemma: convergence re(lambda) zero}
  Let Assumptions~\ref{assumption: ergodicity}, ~\ref{assumption: feature function bounded}, and~\ref{assumption: transition kernel map zero to zero} hold.
  If $\Re(\lambda_i) = 0$, then it holds that for any $j \in \qty{1, \dots, g_i}$ and any $t \geq 0$,
  $\exp(B_{i,j} t) = I_{\rho_{i,j}}$.
\end{lemma}
The proof is in Appendix~\ref{proof: convergence re(lambda) zero}.
Intuitively, if $\Re(\lambda_i)=0$,
we must have $y_{i,j} = 1$.
Otherwise, $\exp(B_{i,j} t)$ cannot be bounded, 
leading to a contradiction with Lemma~\ref{lemma: exp^Bt bounded}.
Furthermore, it must hold that $\Im(\lambda_i) = 0$.
Otherwise, the derivative will not diminish,
leading to a contradiction with Corollary~\ref{corollary: dz/dt diminish}.

Combining Lemmas~\ref{lemma: convergence re(lambda) negative} \&~\ref{lemma: convergence re(lambda) zero}, 
we have $\lim_{t\to\infty} \exp(B_{i,j} t) =I_{\rho_{i,j}} \mathbb{I}\qty{\Re(\lambda_i) = 0}$.
Here, $\mathbb{I}\qty{\cdot}$ is the indicator function.
Therefore, $\lim_{t\to\infty} \exp(At)$ exists, and we define
\begin{align}
  \textstyle
  \label{eq limit of A}
  A_\infty \doteq \lim_{t\to\infty} \exp(At) = P \qty(\bigoplus_{i=1}^k \bigoplus_{j=1}^{g_i} I_{\rho_{i,j}} \mathbb{I}\qty{\Re(\lambda_i) = 0}) P^{-1}.
\end{align}
Plugging this back to~\eqref{eqn: closed form solution}
confirms that for any trajectory $z(t; z_0)$ of the ODE~\eqref{eq shifted ode},
\begin{align}
  \textstyle
  \label{eq: z(t) limit}
  \lim_{t\to\infty} z(t;z_0) = A_\infty z_0.
\end{align}

\begin{proof}[of Theorem~\ref{thm: ode weight convergence}]
Now, we have everything to prove the main statement.
We recall that $w(t;w_0) - w_*$ is a trajectory of~\eqref{eq shifted ode} starting from $w_0 - w_*$.
Therefore,
we have from~\eqref{eq: z(t) limit} 
that
$\lim_{t\to\infty}w(t;w_0) - w_* = A_\infty(w_0 - w_*)$.
In other words,
$\lim_{t\to\infty} w(t;w_0) = w_\infty(w_0) \doteq A_\infty(w_0 - w_*) + w_*$.
Referring to Theorem~\ref{thm: ode value convergence}, there exists $\fS^+ \subseteq \fS$, with $\mu(\fS^+) = 1$, such that
$\lim_{t\to\infty} x(s)^\top w(t; w_0) 
  = x(s)^\top w_\infty(w_0)
  = v_{w_*}(s)$
for all $s \in \fS^+$ and $w_0 \in \R[d]$.
We get $v_{w_\infty(w_0)} = v_{w_*}$ a.e. for any $w_0 \in \R[d]$.
Therefore, we have $w_\infty(w_0) \in \fW_*$ for all $w_0 \in \R[d]$ by Lemma~\ref{lemma: v = v' a.e. iff w' in W},
which completes the proof.
\end{proof}

\subsection{Bounded Invariant Sets}
In the ODE methods for stochastic approximation, if the mean ODE of the stochastic approximation algorithm (cf. ODE~\eqref{eq linear td ode} for linear TD~\eqref{eqn: linear td update}) is not globally asymptotically stable, usually one can only expect that the iterates of the stochastic approximation converge to a bounded invariant set of the ODE.
In light of this, we now study the bounded invariant sets of ODE~\eqref{eq linear td ode}. 
We first study the bounded solutions to the ODE on $(-\infty, +\infty)$.
\begin{theorem}
  \label{theorem: bounded solution constant}
  Let Assumptions~\ref{assumption: ergodicity},~\ref{assumption: feature function bounded}, \&~\ref{assumption: transition kernel map zero to zero} hold.
  Let $w(t)$ be a bounded solution to ODE~\eqref{eq linear td ode} on $(-\infty, +\infty)$, i.e.,
  $\sup_{t\in(-\infty, +\infty)} \norm{w(t)} < \infty$.
  It then holds that $w(t)$ is constant and is in $\fW_*$,
  i.e.,
  there exists some $w_* \in \fW_*$ such that $w(t) = w_*$ holds for any $t \in (-\infty, +\infty)$.
\end{theorem}
\begin{proof}
  Let $z(t; z_0)$ be any bounded solution to~\eqref{eq shifted ode} on $(-\infty, \infty)$. 
  By~\eqref{eqn: exp(At) Jordan normal form} and~\eqref{eq limit of A},
  it holds that, for any $t \in (-\infty, \infty)$,
  \begin{align}
    \exp(At) A_\infty
    =&\textstyle P \qty(\bigoplus_{i=1}^k \bigoplus_{j=1}^{g_i} \exp(B_{i,j}t)) \qty(\bigoplus_{i=1}^k \bigoplus_{j=1}^{g_i} I_{\rho_{i,j}} \mathbb{I}\qty{\Re(\lambda_i) = 0}) P^{-1} \\
    =&\textstyle P \qty(\bigoplus_{i=1}^k \bigoplus_{j=1}^{g_i} I_{\rho_{i,j}} \mathbb{I}\qty{\Re(\lambda_i) = 0}) \qty(\bigoplus_{i=1}^k \bigoplus_{j=1}^{g_i} \exp(B_{i,j}t)) P^{-1} \\
    =&\textstyle A_\infty \exp(At).
  \end{align}
  For any $t \in (-\infty, \infty)$ and $t' > 0$, we then have
  \begin{align}
    \textstyle \norm{z(t; z_0) - A_\infty z_0}
    =& \textstyle\norm{\exp(At)z_0 - A_\infty z_0} \\
    =& \textstyle\norm{\exp(A(t'+t) - At')z_0 - \exp(At')\exp(- At') A_\infty z_0} \\
    =& \textstyle\norm{\exp(A(t'+t) - At')z_0 - \exp(At')A_\infty \exp(- At') z_0} \\
    \leq& \norm{\exp(A(t'+t)) - \exp(At')A_\infty} \norm{\exp(-At')z_0} \\
    =& \textstyle\norm{\exp(A(t'+t)) - \exp(At')A_\infty} \norm{z(-t'; z_0)} \\
    \leq& \norm{\exp(A(t'+t)) - \exp(At')A_\infty}  \sup_{t''\in(-\infty, \infty)}\norm{z(t''; z_0)}\label{eq tmp bound 1}.
  \end{align}
  We note that
  $\lim_{t'\to\infty} \exp(A(t'+t)) - \exp(At')A_\infty = A_\infty - A_\infty A_\infty$.
  In light of~\eqref{eq limit of A}, it holds that
  \begin{align}
    A_\infty A_\infty 
    &= \textstyle P \qty(\bigoplus_{i=1}^k \bigoplus_{j=1}^{g_i} I_{\rho_{i,j}} \mathbb{I}\qty{\Re(\lambda_i) = 0}) P^{-1}
    P \qty(\bigoplus_{i=1}^k \bigoplus_{j=1}^{g_i} I_{\rho_{i,j}} \mathbb{I}\qty{\Re(\lambda_i) = 0}) P^{-1}\\
    &= \textstyle P \qty(\bigoplus_{i=1}^k \bigoplus_{j=1}^{g_i} I_{\rho_{i,j}} \mathbb{I}\qty{\Re(\lambda_i) = 0}) \qty(\bigoplus_{i=1}^k \bigoplus_{j=1}^{g_i} I_{\rho_{i,j}} \mathbb{I}\qty{\Re(\lambda_i) = 0})P^{-1}\\
    &=\textstyle  P \qty(\bigoplus_{i=1}^k \bigoplus_{j=1}^{g_i} I_{\rho_{i,j}} \mathbb{I}\qty{\Re(\lambda_i) = 0}) P^{-1}
    = \textstyle A_\infty.
  \end{align}
  Hence, we obtain $\lim_{t'\to\infty} \exp(A(t'+t)) - \exp(At')A_\infty = 0$.
  Taking $t' \to \infty$ on both sides of~\eqref{eq tmp bound 1} then yields
  $\norm{z(t;z_0) - A_\infty z_0} \le 0 \cdot \sup_{t''\in(-\infty, \infty)}\norm{z(t''; z_0)} = 0$,
  where, to obtain the last equality, we have used the boundedness of $z(t; z_0)$.
  This concludes that for any $t \in (-\infty, \infty)$, we have $z(t;z_0) = A_\infty z_0$.
  Thus, $z(t;z_0)$ is constant.

  We fix any $w_* \in \fW_*$ and recall~\eqref{eq z eq w}.
  Suppose $w(t;w_0)$ is a bounded solution to~\eqref{eq linear td ode} for $ t \in (-\infty, \infty)$.
  Then, $z(t; w_0 - w_*)$ is a bounded solution to~\eqref{eq shifted ode} for $t \in (-\infty, \infty)$.
  We have shown that $z(t; w_0 - w_*)$ is constant whenever it is a bounded solution.
  So, $w(t;w_0)$ is also constant.
  By Theorem~\ref{thm: ode weight convergence}, it holds that $\lim_{t\to\infty} w(t; w_0) = w_\infty(w_0) \in \fW_*$
  for any $w_0 \in \R[d]$.
  Consequently, it must hold that $w(t; w_0) \in \fW_*$ for all $t \in (-\infty, \infty)$ because it is a constant solution. 
\end{proof}
Theorem~\ref{theorem: bounded solution constant} leads to the following characterization of a bounded invariant set.
\begin{corollary}
\label{corollary: bounded invariant set}
  Let Assumptions~\ref{assumption: ergodicity},~\ref{assumption: feature function bounded}, \&~\ref{assumption: transition kernel map zero to zero} hold.
  If $\fW$ is a bounded invariant set of ODE~\eqref{eq linear td ode}, then $\fW \subseteq \fW_*$.
\end{corollary}
The proof is in Appendix~\ref{proof: bounded invariant set}.

\section{Convergence of Linear TD}
\label{section: stochastic approximation}

Having fully characterized the mean ODE~\eqref{eq linear td ode},
we are now ready to connect the linear TD update~\eqref{eq: linear equation} with the mean ODE.
To this end, we consider the joint process $Y_t = (S_t, S_{t+1})$,
which is also a Markov chain.
Recall that $S_t \in \R[K]$.
We then regard $Y_t$ as a vector in $\R[2K]$. 
As a result, the Markov chain $\qty{Y_t}$ evolves in a subset of $\R[2K]$, denoted as $\fY$.
Since $\qty{Y_t}$ is essentially the same chain as $\qty{S_t}$ but viewed differently,
Assumption~\ref{assumption: ergodicity} implies that $\qty{Y_t}$ adopts a unique stationary distribution, referred as $\eta$.
We additionally refer to the transition kernel of $\qty{Y_t}$ as $P_\fY$. 
Given a measurable and integrable function $f: \fY \to \R[d]$, we define $P_\fY f$ as $(P_\fY f)(y) \doteq \int_\fY f(y') P_\fY (\dd{y'} \mid y)$.
With this joint process, 
the linear TD update~\eqref{eqn: linear td update} can be written as
$w_{t+1} = w_t + \alpha_t H(w_t, Y_{t+1}), \text{where } 
 H(w, y) = \qty(r(s) + \gamma w^\top x(s') - w^\top x(s))x(s)$.
Here, we have used shorthand $y \doteq (s, s')$.
The expected update is then $h(w) = \E_{y \sim \eta} \qty[H(w, y)] = Aw + b$.
We now make a few standard assumptions on the behavior of $\qty{Y_t}$.
\begin{xassumption}[Poisson Equation] 
  \label{assumption poisson}
    There exists a function $\nu: \fY \to \R[d]$, such that $\nu_w(y) - (P_\fY \nu_w)(y) = H(w, y) - h(w)$ holds for all $w, y$. Furthermore, there exists a constant $C_\tref{assumption poisson}$ such that $\forall w, y$,
  $\textstyle\norm{\nu_w(y)} \le C_\tref{assumption poisson}(1 + \norm{w})(1 + \norm{y});\quad
   \textstyle\norm{(P_\fY \nu_w)(y) - (P_\fY \nu_{w'})(y)} \le C_\tref{assumption poisson}\norm{w - w'}(1 + \norm{y})$.
\end{xassumption}
\begin{xassumption}[Law of Iterated Logarithm]
  \label{assumption lil}
  For any $w$, there exists a sample-path-dependent finite constant (i.e., a random variable that is finite a.s.) $\zeta_\tref{assumption lil}$, such that\\ 
  $\norm{\sum_{t=1}^n \qty(H(w, Y_t) - h(w))} \leq \zeta_\tref{assumption lil} \sqrt{n \log\log n} \qq{a.s.}$
\end{xassumption}
\begin{remark}
  Both Assumptions~\ref{assumption poisson} \&~\ref{assumption lil} are concerned about the behavior of the Markov chain $\qty{Y_t}$.
  Assumption~\ref{assumption poisson} is the standard way to handle Markovian dependence in the update noise $H(w, Y_t) - h(w)$, while Assumption~\ref{assumption lil} says the cumulative fluctuation of that update noise grows sufficiently slowly almost surely to be absorbed by the diminishing step sizes.
  Furthermore, they are weak in that, if $\fS$ is finite and $\qty{S_t}$ is irreducible and aperiodic, 
  they hold automatically (in fact they hold for any function $H : \R[k] \times \fY \to \R[k]$ and the corresponding expectation $h : \R[k] \to \R[k]$, not just the ones corresponding to linear TD).
  In particular, see Section 8.2.3 of \citet{puterman2014markov} for how Assumption~\ref{assumption poisson} is verified for finite irreducible chains.
  See Remark 3 of \citet{liu2025ode} for how Assumption~\ref{assumption lil} is verified for finite irreducible and aperiodic chains.
  For a generic state space $\fS$, there are multiple sufficient conditions to imply both.
  For example, if $\qty{Y_t}$ is positive Harris\footnote{See pages 235 and 204 of \citet{meyn2012markov} for the definitions of positive and Harris chains, respectively.} and satisfies the Lyapunov drift condition (V4)\footnote{See page 386 of \citet{meyn2012markov} for the definition of the (V4) condition.} with a constant function $V(y) = 1 \, \forall y$,\footnote{This condition again trivially holds for finite irreducible and aperiodic chains.}
  Assumption~\ref{assumption poisson} can be satisfied via the boundedness of the fundamental kernel through Theorem 2.3 of \citet{glynn1996liapounov}.
  Assumption~\ref{assumption lil} can be satisfied via $V$-uniform ergodicity\footnote{See page 387 of \citet{meyn2012markov} for the definition of $V$-uniform ergodicity.} (which is implied by (V4) through Theorem 16.0.1 (iv) of \citet{meyn2012markov}) through Theorem 17.0.1 (iii) and (iv) of \citet{meyn2012markov}.
  Such implications are standard and have been extensively studied in the existing literature. 
  But detailing them necessitates introducing quite a few additional definitions, 
  which is not our contribution and deviates from the goal of this paper. 
  So, we present the current forms of Assumption~\ref{assumption poisson} \&~\ref{assumption lil} directly.
  They are also widely used in existing works.
  For example, Assumption~\ref{assumption poisson} is the most important assumption in~\citet{benveniste1990MP}.
  Assumption~\ref{assumption lil} is also used in \citet{liu2025ode}.
\end{remark}
We additionally make the following assumption regarding the state space.
\begin{xassumption}
  \label{assumption: state space compact}
  $\fS$ is compact.  
\end{xassumption}
  \begin{remark}
    We note that Assumption~\ref{assumption: state space compact} is not needed to characterize the TD fixed point or the mean ODE in Sections~\ref{section: TD fixed points} \&~\ref{section: ode analysis}.
    It is also unnecessary for linear TD itself.
    We employ Assumption~\ref{assumption: state space compact} as a convenient sufficient condition to satisfy the assumptions made in~\citet{benveniste1990MP} to work with infinite state spaces.
    That being said,
    we do envision that Assumption~\ref{assumption: state space compact} can be relaxed to some non-compact state space with some additional assumptions,
    as long as the assumptions in~\citet{benveniste1990MP} are satisfied.
    This relaxation is standard and too technical to be included in this paper, but it is worth noting that it is possible.
  \end{remark}
The next theorem demonstrates the stability of linear TD.
\begin{theorem}
  \label{theorem: w_t a.s. bounded}
  Let Assumptions \ref{assumption: learning rate}, \ref{assumption: ergodicity}, \ref{assumption: feature function bounded}, \ref{assumption: transition kernel map zero to zero}, \ref{assumption poisson}, \& \ref{assumption: state space compact} hold.
  Then, the iterates $\qty{w_t}$ generated by \eqref{eqn: linear td update} is stable, i.e., $\sup_t \norm{w_t} < \infty \qq{a.s.}$
\end{theorem}
The proof is in Appendix~\ref{proof: w_t a.s. bounded}.
It is based on Theorem 17(a) of \citet{benveniste1990MP},
where we use $U(w) \doteq \norm{w - w_*}^2 + \norm{w_*}^2$
as an energy function.
Here, $w_*$ is any fixed point in $\fW_*$.
In the canonical analysis with linearly independent features,
$\norm{w - w_*}^2$ is commonly used as the Lyapunov function.
Without assuming linear independence,
we additionally add $\norm{w_*}^2$ such that $U(w) \geq \frac{1}{2}\norm{w}^2$ always holds.
Apparently, this $U(w)$ is not a Lyapunov function,
but it is sufficient to prove stability per \citet{benveniste1990MP}.
In particular,
this $U(w)$ satisfies Conditions (i) and (ii) on page 239 of \cite{benveniste1990MP}.

Having established the stability,
the convergence of $\qty{w_t}$ to a bounded invariant set is now expected from standard stochastic approximation results \citep{benveniste1990MP,kushner2003stochastic,borkar2009stochastic,liu2025ode}.
Here, we use Corollary 1 of \citet{liu2025ode} to establish the desired convergence,
which is essentially a simplified version of Theorem~1 in Chapter 5 of \citet{kushner2003stochastic}.

\begin{theorem}
\label{thm: linear TD converge to invariant set}
  Suppose Assumptions \ref{assumption: learning rate}, \ref{assumption: ergodicity}, \ref{assumption: feature function bounded}, \ref{assumption: transition kernel map zero to zero}, \ref{assumption poisson}, \ref{assumption lil}, \& \ref{assumption: state space compact} hold.
  Let $w_* \in \fW_*$.
  Then, for any $w_0\in\R[d]$, there exists $\fS^+ \subseteq \fS$ with $\mu(\fS^+) = 1$, such that the iterates $\qty{w_t}$ generated by \eqref{eqn: linear td update} satisfy $\forall s \in \fS^+$,
  $\lim_{t\to\infty} x(s)^\top w_t = x(s)^\top w_* \qq{a.s.}$
\end{theorem}
The proof is in~\ref{proof: linear TD converge to invariant set}.
We conclude this section with an open problem. Theorem~\ref{thm: linear TD converge to invariant set} is in parallel with Theorem~\ref{thm: ode value convergence} in that both are concerned with the weight convergence to a set.
In light of Theorem~\ref{thm: ode weight convergence},
a natural question arises: 
can we prove that $\qty{w_t}$ converges to a single (possibly sample path dependent) point?
Unfortunately,
we do not have a definite answer now.
The best we know is that from Theorem~\ref{thm: linear TD converge to invariant set},
any convergent subsequence of $\qty{w_t}$ will converge to some sample path dependent point in $\fW_*$.
We further show that $\qty{w_t}$ exhibits some local stability along a convergent subsequence in the following sense.
\begin{corollary}
\label{corollary: local stability}
  Let Assumptions \ref{assumption: learning rate}, \ref{assumption: ergodicity}, \ref{assumption: feature function bounded}, \ref{assumption: transition kernel map zero to zero}, \ref{assumption poisson}, \ref{assumption lil}, \& \ref{assumption: state space compact} hold.
  Then, there exists at least one convergent subsequence of $\qty{w_t}$, denoted as $\qty{w_{t_k}}$,
  such that
  for any $T < \infty$, it holds that 
  $\lim_{k\to\infty} \max_{t_k \le j \le m(t_k, T)} \norm{w_j - w_*} = 0$,
  where $w_* \in \fW_*$ is the limit of $\qty{w_{t_k}}$.
\end{corollary}
Here, $m(t, T)$ is defined as 
$m(t, T) \doteq \max\qty{n \middle| \sum_{i = t}^n \alpha_i \leq T}$.
Intuitively, the magnitude of the updates of linear TD in~\eqref{eqn: linear td update} is controlled by the learning rate $\alpha_i$.
Then, $\qty{t, t+1, \dots, m(t, T)}$ denotes a period during which the total magnitude of updates is no more than $T$. 
The proof is in~\ref{proof: local stability}.
Corollary~\ref{corollary: local stability} essentially confirms that $\qty{w_t}$ will visit (arbitrarily small) neighbors of $w_*$ infinitely many times.
The number of update steps during each visit (i.e., $m(t_k, T) - t_k$) diverges to $\infty$.

\section{Finite State Space}
\label{section: finite state space}
We now consider a special case where $S$ is finite,
which is also considered in many previous works~\citep{sutton1988learning,dayan1992tdlambda,tsitsiklis1997analysis}.
This special case allows us to make a more detailed comparison with previous approaches and 
provide finer characterization of the set of TD fixed points $\fW_*$.

When $\fS$ is finite, we can represent the transition dynamics as a stochastic matrix $P \in \qty[0,1]^{\ns\times\ns}$, where $P(i, j) = p\qty(s_j \mid s_i)$.
Similarly, we overload $r$ to denote the vector representation of the reward function when it does not confuse.
We now define the Bellman operator $\bop: \R[\ns] \to \R[\ns]$ as $\bop v \doteq r + \gamma P v$, given $v \in \R[\ns]$.
We shall also represent the feature function $x$ compactly as a matrix $X \in \R[\ns \times d]$, where $x(s)$ is the $s$-th row of $X$.
Remarkably,
Assumptions \ref{assumption: ergodicity}, \ref{assumption: feature function bounded}, \ref{assumption: transition kernel map zero to zero}, \ref{assumption poisson}, \ref{assumption lil}, \& \ref{assumption: state space compact}
trivially hold for finite $\fS$ when $\qty{S_t}$ is irreducible and aperiodic.
We also define $D \in \R[\ns\times\ns]$ as the diagonal matrix whose diagonal terms are the stationary distribution of the induced Markov chain $\qty{S_t}$.
Under Assumption~\ref{assumption: ergodicity}, $D$ is symmetric and positive definite.

We now review the canonical convergence analysis of linear TD (cf. \citet{tsitsiklis1997analysis}) under linearly independent features (i.e., $X$ has a full column rank).
Given any vector $v \in \R[\ns]$, define the projection of $v$ onto the column space of $X$ as
\begin{align}
  \label{eq canonical projection}
  \textstyle\Pi v \doteq \arg\min_{v_0 \in \qty{Xw \middle| w \in \R[d]}} \Dnorm{v_0 - v}^2
  = X \arg\min_{w\in\R[d]} \norm{Xw - v}_D^2.
\end{align}
Suppose the features are linearly independent, i.e., $X$ has linearly independent columns, \citet{tsitsiklis1997analysis} prove that
$\Pi = X \qty(X^\top D X)^{-1} X^\top D$.
The fact that $X^\top D X$ is nonsingular follows directly from the linear independence of the features.
The composition $\Pi \bop$ of the projection operator and the Bellman operator is a contraction mapping with respect to $\Dnorm{\cdot}$~\citep{tsitsiklis1997analysis} and thus adopts a unique fixed point $v_*$, such that
\begin{align}
  \label{eq canonical v star}
  \textstyle\Pi \bop v_* = v_*
\end{align}
according to the Banach fixed-point theorem.
Because of feature linear independence, there exists a unique $w_*$, such that $Xw_* = v_*$.
This $w_*$ has the following two remarkable properties.
First, $w_*$ is the unique zero of the mean squared projected Bellman error (MSPBE)~\citep{sutton2018reinforcement}.
Here, MSPBE is defined as
$\text{MSPBE}(w) \doteq \Dnorm{\Pi(\bop X w - X w)}^2$.
Second,
$w_*$ is the unique solution to linear system~\eqref{eq: linear equation},
where $A = X^\top D \qty(\gamma P - I)X$ and $b = X^\top D r$~\citep{bertsekas1996neuro}.
Notably, $A$ and $b$ here are defined in the same way as~\eqref{eq: A definition} and~\eqref{eq: b definition}.
The uniqueness of these two properties 
follows from the fact that $A$ is negative definite when features are linearly independent~\citep{bertsekas1996neuro,tsitsiklis1997analysis}.
Under Assumptions~\ref{assumption: learning rate}, \ref{assumption: ergodicity}, and the feature linear independence assumption, \citet{tsitsiklis1997analysis} prove that the iterates $\qty{w_t}$ generated by \eqref{eqn: linear td update} satisfy $\lim_{t\to\infty} w_t = w_*$ a.s.
As a result,
the weight $w_*$ is the unique TD fixed point.
Importantly,
the negative definiteness of $A$ ensures that the ODE~\eqref{eq linear td ode} is globally asymptotically stable,
which is key to the convergence proof of \citet{tsitsiklis1997analysis}.
In the above convergence analysis,
the linear independence of the features is vital in the following three aspects:
\begin{enumerate}[(i)]
  \item it ensures that the TD fixed point is unique;
  \item it ensures that ODE~\eqref{eq linear td ode} is well-behaved;
  \item it ensures that TD update~\eqref{eqn: linear td update} can be properly related to ODE~\eqref{eq linear td ode}.
\end{enumerate}
Consequently,
removing the feature linear independence assumption would at least entail three corresponding challenges.
We addressed each of them in the previous sections.
Namely,
in Section~\ref{section: TD fixed points},
we analyzed TD fixed points with arbitrary features.
In Section~\ref{section: ode analysis},
we analyzed ODE trajectories with arbitrary features.
In Section~\ref{section: stochastic approximation},
we established the convergence of linear TD with arbitrary features.

Next, we show how we can relate $\fW_*$ with the MSPBE.
If $X$ does not have full column rank,
the $\arg\min$ in~\eqref{eq canonical projection} may return a set of weights instead of a unique one.
In light of this, we redefine $\Pi$ to always select the weight with the smallest norm.
Namely, 
we redefine $\Pi$ as
\begin{align}
  \label{eq new projection}
  \textstyle \Pi v \doteq X \arg \min_{w \in \arg \min_w \Dnorm{Xw - v}^2} \norm{w}.
\end{align}
In the rest of the section,
we always use this more general definition of $\Pi$.
It turns out that this new definition of $\Pi$ also enjoys a closed-form expression.
\begin{lemma}
  \label{lemma: projection matrix}
  Let Assumption~\ref{assumption: ergodicity} hold. 
  Let $(\cdot)^\dagger$ denote the pseudo-inverse. Then, we have
  $\Pi = X (D^{1/2}X)^\dagger D^{1/2}$.
\end{lemma}
The proof is provided in Appendix~\ref{proof: projection matrix},
where we have used standard results from least squares, see, e.g., \citet{gallier2019algebra}.
Moreover,
as expected,
this new definition of $\Pi$ preserves the desired contraction property.
\begin{lemma}
  \label{lemma: contraction operator}
  Let Assumption~\ref{assumption: ergodicity} hold. 
  Then, $\Pi\bop$ is a contraction operator w.r.t. $\norm{\cdot}_D$.
\end{lemma}
The proof is provided in Appendix~\ref{proof: contraction operator}.
Banach's fixed point theorem ensures that $\Pi\bop$ adopts a unique fixed point,
referred to as $v_*$,
such that
\begin{align}
  \label{eq v star}
  \textstyle\Pi\bop v_* = v_*.
\end{align}
Here, we have overloaded the definition of $v_*$ in~\eqref{eq canonical v star} 
and in the rest of the section we shall use the overloaded definition of $v_*$.
The definition of $\Pi$ in~\eqref{eq new projection} then immediately ensures that there exists at least one $w_*$, such that 
\begin{align}
  \label{eq w star}
  \textstyle Xw_* = v_*,
\end{align}
implying $\Pi \bop Xw_* = Xw_*$.

The next lemma mirrors Lemma~\ref{lemma: v = v' a.e. iff w' in W} in the case of finite state space.
The only difference is that we now have exact equivalence instead of almost everywhere equivalence for the approximated value functions.
\begin{lemma}
  \label{lemma: value equivalence for TD fixed points}
  Let Assumption~\ref{assumption: ergodicity} hold.
  Then, for any $w, w' \in \fW_*$, 
  we have $Xw = Xw'$.
\end{lemma}
The proof is provided in Appendix~\ref{proof: value equivalence for TD fixed points}.
Having redefined $\Pi$ and established the equivalence between value estimates for TD fixed points,
we are able to relate $\fW_*$ to the MSPBE with the following theorem.
\begin{theorem}
\label{thm: td fixed point and mspbe}
  Let Assumption~\ref{assumption: ergodicity} hold.
  Then, $\forall w\in\R[d]$,
  $Aw + b = 0 \iff \Pi\bop Xw = Xw$.
\end{theorem}
The proof is provided in Appendix~\ref{proof: td fixed point and mspbe}.
The above equivalence implies that the $w_*$ defined in~\eqref{eq w star} must satisfy $Aw_* + b = 0$, i.e., $w_* \in \fW_*$.
It also confirms that all weights in $\fW_*$ minimize MSPBE.

\section{Related Work}
The seminal work of~\citet{sutton1988learning} formalizes the idea of temporal difference learning.
The linear TD update~\eqref{eqn: linear td update} in this paper is referred to as TD(0) in~\citet{sutton1988learning},
which is a special case of the more general TD algorithm with eligibility trace,
referred to as TD($\lambda$) in~\citet{sutton1988learning}.
\citet{sutton1988learning} proves the convergence of linear TD(0) in expectation.
Extending Sutton's work, 
\citet{dayan1992tdlambda} proves the convergence of linear TD($\lambda$) in expectation for a general $\lambda$ and the almost sure convergence of tabular TD(0).
Later, \citet{dayan1994tdlambda} further show the almost sure convergence of linear TD($\lambda$).
One should note that the works of~\citet{sutton1988learning,dayan1994tdlambda} 
require the observations to be linearly independent, i.e., the feature matrix $X$ has full row rank; in most cases, this is an even stronger assumption than having linearly independent columns, because full row rank is essentially equivalent to a tabular representation.
Furthermore, the version of TD($\lambda$)~\citet{sutton1988learning,dayan1992tdlambda,dayan1994tdlambda} 
consider is ``semi-offline'', where the weight update occurs after each sequence of observations rather than at every step.
\citet{tsitsiklis1997analysis} provide the first proof of almost sure convergence of linear TD($\lambda$), assuming linearly independent features. 
The linear TD considered in this paper is exactly the same as \cite{tsitsiklis1997analysis} with $\lambda = 0$.
\citet{tadic2001convergence} proves the almost sure convergence of linear TD with weaker assumptions than \citet{tsitsiklis1997analysis} but still requires linearly independent features.
\citet{tadic2001convergence} argues (without concrete proof) that without the linear independence assumption,
the projected iterates $\qty{\Gamma w_t}$ converges almost surely,
where $\qty{w_t}$ are generated by~\eqref{eqn: linear td update} and $\Gamma$ projects a vector into the row (column) space of $X^\top DX$ as $\Gamma(w) = \arg\min_{w' \in \qty{X^\top DX z | z \in \R[d]}} \norm{w' - w}^2 
= X^\top DX(X^\top DX)^\dagger w$.
Without assuming linearly independent features,
the convergence of $\qty{\Gamma w_t}$ does not necessarily imply the convergence of $\qty{Xw_t}$.
More recently, \citet{brandfonbrener2020geometric} sought to extend the convergence guarantee of linear TD to nonlinear function approximators.
They study the associated ODE of nonlinear function approximators and show that, under certain homogeneous assumptions of the function approximator, the $\liminf$ of the norm of the approximated value function is finite.
Under stricter assumptions between the gradient of the function approximator and the reversibility\footnote{Reversibility measures how symmetric the matrix $D\qty(I - \gamma P)$ is.} of the transition dynamics, the function approximator converges to the true value function.
\citet{cai2024neural} provide a finite-sample analysis of TD with a two-layer overparameterized neural network.
Neural TD, their proposed TD algorithm, requires a projection operator that confines the network's weights to a ball centered on the initial weights and uses weight averaging.
Under those modifications,
\citet{cai2024neural} establish the value convergence rate of Neural TD in expectation,
including an error term that only diminishes when the width of the network goes to infinity.
Despite working only with linear function approximation, we do not modify the original linear TD algorithm and still provide almost sure convergence.
We envision that our results will shed light on an almost sure convergence of Neural TD,
which we leave for future work.

The assumptions made about the features are not exclusive to the asymptotic analysis of linear TD.
They are also prevalent in the study of finite-sample guarantees.
Recent works on characterizing the convergence rate of linear TD include~\citet{bhandari2018finite,lakshminarayanan2018linear,srikant2019finite,chen2025concentration,mitra2025finite}.
All their analyses rely on the linear independence of the features.
As a matter of fact, almost all previous analysis of RL algorithm with linear function approximation assume feature linear independence, see, e.g.,~\citet{sutton2008gtd,sutton2009fast,maei2011gradient,hackman2013faster,liu2015finite,yu2015convergence,yu2016weak,zou2019finite,yang2019provably,zhang2020gradientdice,zhang2020provably,xu2020improving,xu2020non,wu2020finite,chen2021finitesample,yang2021onconvergence,qiu2021finite,zhang2021average,zhang2021breaking,xu2021doubly,zhang2022global,zhang2022truncated,zhang2023convergence,chen2023global,dalfabbro2024finite,wang2024finitetime,ganesh2025orderoptimal,liu2025ode,qian2025revisiting,maity2025adversariallyrobust,peng2025afinitesample,liu2025linearq}.
The only exception we are aware of is~\citet{xie2025finite}, who build on our results and conduct a finite-sample analysis of linear TD under arbitrary features to establish $L^2$ convergence rates for linear TD in both discounted and average-reward settings.

\section{Conclusion}
This work contributes to RL with arbitrary features using linear TD as an example, 
where the commonly used linear independence assumption on features is lifted.
The insight and techniques in this work can be easily used to analyze other linear RL algorithms,
e.g.,
linear SARSA \citep{rummery1994line,zou2019finite,zhang2023convergence},
gradient TD methods \citep{sutton2008gtd,sutton2009fast,maei2011gradient,zhang2021average,qian2025revisiting},
emphatic TD methods \citep{yu2015convergence,sutton2016emphatic,zhang2022truncated}, 
density ratio learning methods \citep{nachum2019dualdice,zhang2020gradientdice},
TD with target networks \citep{lee2019target,carvalho2020new,zhang2021breaking},
linear $Q$-learning \citep{meyn2024projected,liu2025linearq,liu2025extensions},
and
actor-critic methods with compatible features \citep{sutton2000policy,konda2000actor,zhang2020provably},
as well as their in-context learning version \citep{wang2025transformers,wang2025towards}. 
The mode of convergence may also be extended to high probability convergence with exponential tails (and thus $L^p$ convergence) following the techniques in~\citet{chen2025concentration,qian2024almost}. 
This work is also closely related to overparameterized neural networks,
where the linearization of the neural network at the initial weights naturally results in features that are not necessarily linearly independent \citep{cai2024neural}.
Recently,
the convergence of linear TD with linearly independent features has been formally verified in Lean \citep{zhang2025towards}.
We envision that the convergence of linear TD with arbitrary features can also be formally verified in Lean based on the results in this paper.

That said, the major open question is whether we can prove that the linear TD weight iterates converge to a possibly sample-path-dependent TD fixed point.
We are optimistic about this question because recent work \citet{blaser2026asymptotic} show that tabular average-reward TD can converge to a sample-path-dependent TD fixed point almost surely.
We are, however, pessimistic about solving this problem with the ODE technique.
Instead, recent works on stochastic Krasnoselskii-Mann iterations based on a fox-and-hare model \citep{cominetti2014rate,bravo2019rates,bravo2024stochastic,blaser2026asymptotic} may be a promising direction to solve this problem.

\acks{This work is supported in part by the US National Science Foundation under the awards III-2128019, SLES-2331904, and CAREER-2442098, the Commonwealth Cyber Initiative's Central Virginia Node under the award VV-1Q26-001, and a Cisco Faculty Research Award.}

\appendix

\section{Mathematical Background}
We first provide the definition of the Moore-Penrose pseudo-inverse for completeness.
\begin{xdefinition}
  (Definition 23.1 of \cite{gallier2019algebra})
  Given any nonzero $m \times n$ matrix $A$ of rank $r$, if $A = V \Sigma U^\top$ is a singular value decomposition of $A$ such that
  $\Sigma = \mqty[\Lambda & 0_{r, n-r}\\
                   0_{m-r, r} & 0_{m-r, n-r}], \text{where } 
    \Lambda = \mqty[\dmat{\lambda_1, \ddots, \lambda_r}]$
  is an $r \times r$ diagonal matrix consisting of the nonzero singular values of $A$, then if we let $\Sigma^\dagger$ be the $n \times m$ matrix
  $\Sigma^\dagger = \mqty[\Lambda^{-1} &  0_{r, m-r}\\
    0_{n-r, r} & 0_{n-r, m-r}], \text{where }
    \Lambda^{-1} = \mqty[\dmat{1/\lambda_1, \ddots, 1/\lambda_r}]$,
  the pseudo-inverse of $A$ is defined as $A^\dagger = U \Sigma^\dagger V^\top$.
  If $A = 0_{m,n}$ is the zero matrix, the pseudo-inverse of $A$ is defined as $A^\dagger = 0_{n,m}$.
 \end{xdefinition}

\begin{xtheorem}
  \label{thm: least squares}
  (Theorems 23.1 \& 23.2 of \cite{gallier2019algebra})
  For any matrix $A$ and any vector $b$,
  consider the least square problem $\min_x \norm{Ax - b}$.
  Then, $x^\dagger \doteq A^\dagger b$ is a minimizer.
  Furthermore, any other possible minimizer has a strictly larger norm than $x^\dagger$,
  i.e., $x^\dagger$ is the unique minimizer of the problem
  $\min\qty{\norm{x} \middle| \norm{Ax - b} = \min_{y} \norm{Ay - b}}$.
\end{xtheorem}

\begin{xlemma}
  \label{lemma: AA norm}
  For any matrix $A \in \R[m \times n]$ of rank $r$, $\norm{A A^\dagger} \le 1$.
\end{xlemma}
\begin{proof}
  This follows immediately from the fact that $AA^\dagger$ is an orthogonal projection onto the range of $A$ (see, e.g., Proposition 23.4 of \cite{gallier2019algebra})
  and the well-known fact that the operator norm of an orthogonal projection is $1$.
  We also provide a short proof for completeness.
  Let $A = V \Sigma U^\top$ be the singular value decomposition. 
  We then have
  $\norm{AA^\dagger} = \norm{V \Sigma U^\top (U \Sigma^\dagger V^\top)} = \norm{V Q V^\top},
  \text{where } 
  Q = \mqty[I_r & 0_{r, m-r} \\
            0_{m-r, r} & 0_{m-r, m-r}]$.
  It then follows easily that 
  $\norm{AA^\dagger} \le \norm{V} \norm{Q} \norm{V^\top} \leq 1$.
\end{proof}
We then list supporting theorems from~\citet{khalil2002nonlinear} about nonlinear systems.
\begin{xtheorem}
  \label{thm: globally asymptotically stable}
  (Theorem 4.2 of~\citet{khalil2002nonlinear})
  Let $z = z_*$ be an equilibrium point for~\eqref{eq shifted ode}.
  Let $U: \R[d] \to \R$ be a continuously differentiable function such that
  \begin{enumerate}
    \item $U(z_*) = 0$ and $U(z) > 0, \forall z \neq z_*$;
    \item $\norm{z} \to \infty \implies U(z) \to \infty$;
    \item $\dv{U(z)}{t} < 0, \forall z \neq z_*$,
  \end{enumerate}
  then $z = z_*$ is globally asymptotically stable.
\end{xtheorem}

\subsection{Stochastic Approximation}
We then present some general results in stochastic approximation from \citet{benveniste1990MP} and \citet{liu2025ode}.
Consider the iterates $\qty{w_t}$ in $\R[d]$ generated by
\begin{align}
\label{eq: SA general form}
  \textstyle
  w_{t+1} = w_t + \alpha_t H(w_t, Y_t),
\end{align}
where $\qty{Y_t}$ is a Markov chain evolving in $\R[k]$ and the function $H$ maps from $\R[d] \times \R[k]$ to $\R[d]$.
We use $P_\fY: \fB\qty(\R[k]) \times \R[k] \to [0, 1]$ to denote the transition kernel of $\qty{Y_t}$.

We first present a result from \cite{benveniste1990MP} concerning the stability of $\qty{w_t}$.
We note that \cite{benveniste1990MP} considers a time-inhomogeneous Markov chain,
but our $\qty{Y_t}$ is time-homogeneous.
Therefore, some assumptions in \cite{benveniste1990MP} trivially hold in our setting.
For simplicity, we list only the nontrivial assumptions here.

\begin{xassumption}
  \label{assumption: learning rate sa}
    $\{\alpha_t\}$ is a decreasing sequence of positive real numbers such that
    \begin{align}
      \textstyle 
      \sum_{t=0}^\infty \alpha_t = \infty,
      \quad\sum_{t=0}^\infty \alpha_t^2 < \infty,
      \quad\lim_{t\to\infty} \qty(\frac{1}{\alpha_{t+1}} - \frac{1}{\alpha_t}) < \infty.    
    \end{align}
  \end{xassumption}
  
  \begin{xassumption}
  \label{assumption: Y_t ergodic}
  The Markov chain $\qty{Y_t}$ admits a well-defined and unique stationary distribution $\eta: \fB(\R[k]) \to [0,1]$, such that $\eta(Z) = \int_{\R[k]} P_\fY(Z \mid y)\eta(\dd{y})$ for all $Z \in \fB(\R[k])$ and $\eta(U) > 0$ for all non-empty open sets $U \subseteq \R[k]$.
  \end{xassumption}
  
  \begin{xassumption}
  \label{assumption: h bounded}
  There exists a constant $K_\tref{assumption: h bounded} \in \R$, 
  such that for all $w \in \R[d]$ and $y \in \R[k]$, we have
  $\norm{H(w,y)} \le K_\tref{assumption: h bounded}(1 + \norm{w})(1+\norm{y})$.
  \end{xassumption}
  For any function $f(w, y)$ on $\R[d] \times \R[k]$, we shall denote the partial mapping $y \to f(w, y)$ by $f_w$.
  We shall also denote the function $y \to \int_{\R[k]} f(w, y') P_\fY(\dd{y'} \mid y)$
  as $P_\fY f_w$.
  \begin{xassumption}
  \label{assumption: g and nu exist}
    There exists a function $g: \R[d] \to \R[d]$, and for each $w \in \R[d]$ a function $\nu_w: \R[k] \to \R[d]$, such that
    \begin{enumerate}[(i)]
        \item $g$ is locally Lipschitz on $\R[d]$;
        \item $\nu_w - P_\fY\nu_w = H_w - g(w)$ for all $w \in \R[d]$;
        \item there exists a constant $K_\tref{assumption: g and nu exist} \in \R$, such that for all $w \in \R[d]$ and $y \in \R[k]$,
        $\norm{\nu_w(y)} \le K_\tref{assumption: g and nu exist} (1 + \norm{w}) (1 + \norm{y})$;
        \item there exists a constant $K'_\tref{assumption: g and nu exist} \in \R$, such that for all $w, w' \in \R[d]$ and $y \in \R[k]$,\\
        $\norm{(P_\fY\nu_w)(y) - (P_\fY\nu_{w'})(y)} \le K'_\tref{assumption: g and nu exist}\norm{w - w'}(1 + \norm{y})$.
    \end{enumerate} 
  \end{xassumption}
  \begin{xassumption}
    \label{assumption: bounded moment}
    For all $a \in \R[d]$, $q > 0$, and $n \ge 0$, there exists $K_\tref{assumption: bounded moment}(q) < \infty$, such that
    $\E\qty[1 + \norm{Y_{n+1}}^q \mid w_0 = a, Y_0 = y] \le K_\tref{assumption: bounded moment}(q) \qty(1 + \norm{y}^q)$.
  \end{xassumption}
   Define $h(w) \doteq \E_{Y \sim \eta}\qty[H(w, Y)]$.
   Then, we have the following theorem in~\cite{benveniste1990MP}.
  \begin{xtheorem}
    \label{theorem: w_t a.s. bounded conditions}
    (Theorem 17(a) of \cite{benveniste1990MP})
    Let Assumptions~\ref{assumption: learning rate sa} and \ref{assumption: h bounded} - \ref{assumption: bounded moment} hold.
    Assume there exists a function $U: \R[d] \to [0, \infty)$ of class $C^2$ with bounded second derivatives\footnote{A function of class $C^2$ is a function whose second derivative is continuous in its domain.}  and a constant $c > 0$, such that for all $w\in\R[d]$,
    \begin{enumerate}
        \item[i.] $\innerdot{\nabla_w U(w)}{h(w)} \le 0$, \text{and}
        \item[ii.] $U(w) \ge c\norm{w}^2$,
    \end{enumerate}
    then for all $w_0 \in \R[d]$, the iterates $\qty{w_t}$ generated by~\eqref{eq: SA general form} is stable, i.e.,
    $\sup_t \norm{w_t} < \infty \qq{a.s.}$
  \end{xtheorem}
We now present results from \cite{liu2025ode} concerning the convergence of $\qty{w_t}$. 
\begin{xassumption}
\label{assumption: L(y) existence}
  There exists a function $L: \fY \to \R$ such that for any $w, w', y$,\\
  $\norm{H(w,y) - H(w',y)} \le L(y)\norm{w - w'}$.
  Moreover, the following expectations are well-defined and finite for any $w \in \R[d]$:
  $h(w) \doteq \E_{Y \sim \eta} \qty[H(w ,Y)], L \doteq \E_{Y \sim \eta}\qty[L(Y)]$.
\end{xassumption}
\begin{xassumption}
\label{assumption: w_t bounded}
  For all $w_0 \in \R[d]$, the iterates $\qty{w_t}$ generated by~\eqref{eq: SA general form} is stable, i.e., $\sup_t \norm{w_t} < \infty$ a.s.
\end{xassumption}
\begin{xtheorem}
  \label{thm: shuze converge to bounded invariant set}
  (Corollary 8 of~\citet{liu2025ode})
  Let Assumptions \ref{assumption: learning rate sa}, \ref{assumption: Y_t ergodic}, \ref{assumption: L(y) existence} and \ref{assumption: w_t bounded} hold.
  Then the iterates $\qty{w_t}$ generated by~\eqref{eq: SA general form} converge almost surely to a (possibly sample path dependent) bounded invariant set of the ODE
  $\dv{w(t)}{t} = h(w(t))$.
\end{xtheorem}
We note that the original form of Corollary~8 of \cite{liu2025ode} involves some additional assumptions regarding $H_\infty(w, y) \doteq \lim_{c\to\infty} \frac{H(cw, y)}{c}$ and $h_\infty(w) \doteq \E_{Y\sim\eta}\left[H_\infty(w, Y)\right]$,
as well as assumptions regarding the ODE~$\dv{w(t)}{t} = h_\infty(w(t))$.
Those assumptions are related to the ODE@$\infty$ technique to establish the stability of the iterates $\qty{w_t}$.
We refer the reader to \citet{borkar2009stochastic,borkar2025ode,liu2025ode} for more details about this technique.
After establishing stability,
the convergence part in Corollary 8 of \cite{liu2025ode} follows a standard approach based on the Arzela-Ascoli theorem, akin to Theorem 1 of Chapter 5 of \cite{kushner2003stochastic} and does not rely on those additional assumptions. 
In this paper,
we instead assume stability directly (Assumption~\ref{assumption: w_t bounded}).
As a result,
we no longer need those ODE@$\infty$ related assumptions and thus omit them.
In this work,
instead of using the ODE@$\infty$ technique,
we will use the Lyapunov method in \citet{benveniste1990MP} to establish stability.

\section{Proofs in Section~\ref{section: background}}
\subsection{Proof of Lemma~\ref{lemma: pv norm}}
\label{proof: pv norm}
\begin{proof}
  By definition, we have
  \begin{align}
    \norm{P_\pi v}^2_\mu
    &= \textstyle\int_{s \in \fS} \qty((P_\pi v)(s))^2 \mu(\dd{s})
    = \textstyle\int_{s \in \fS} \qty(\int_{s' \in \fS} v(s') p(\dd{s'}|s))^2 \mu(\dd{s})\\
    &\le \textstyle\int_{s \in \fS} \int_{s' \in \fS} v(s')^2 p(\dd{s'}|s) \mu(\dd{s}) \explain{Jensen's inequality}\\
    &=\textstyle\int_{s' \in \fS} v(s')^2 \int_{s \in \fS}   p(\dd{s'}|s) \mu(\dd{s})
    =\textstyle\int_{s' \in \fS} v(s')^2 \mu(\dd{s'})
    =\textstyle\norm{v}^2_\mu.
  \end{align}
\end{proof}

\subsection{Proof of Lemma~\ref{lemma: A NSD}}
\label{proof: A NSD}
\begin{proof}
  We prove the negative semi-definiteness of $A$ by showing $\indot{w}{Aw} \le 0$ for any $w\in\R[d] \setminus \qty{0}$.
  Firstly, we define $v_w(s) \doteq x(s)^\top w$ for all $s \in \fS$.
  In addition, $v_w \in L_2(\fS, \mu)$ by Assumption~\ref{assumption: feature function bounded}.
  Then, we have
  \begin{align}
    \indot{w}{Aw} 
    =& \textstyle\indot{v_w}{\gamma P_\pi v_w - v_w} \explain{By~\eqref{eq: A definition}}_\mu\\
    =& \textstyle\indot{v_w}{\gamma P_\pi v_w}_\mu - \norm{v_w}^2_\mu
    \le \textstyle\gamma \norm{v_w}_\mu \norm{P_\pi v_w}_\mu - \norm{v_w}^2_\mu\\
    \le& \textstyle(\gamma - 1)\norm{v_w}^2_\mu \explain{Lemma~\ref{lemma: pv norm}}
    \le \textstyle 0.
  \end{align}
\end{proof}

\section{Proofs in Section~\ref{section: TD fixed points}}
\subsection{Proof of Lemma~\ref{lemma: inner integral positive definite}}
\label{proof: inner integral positive definite}
\begin{proof}
  We first note that $v_w \in L_2(\fS, \mu)$ because $\norm{x}$ is bounded under Assumption~\ref{assumption: feature function bounded}.
  Then, we have $w^\top A w = \indot{v_w}{\gamma P_\pi v_w - v_w}_\mu$.
  We begin with the direction $w^\top A w = 0 \impliedby v_w = 0$ a.e.
  When $v_w = 0$ a.e., we have $w^\top A w = \indot{v_w}{\gamma P_\pi v_w - v_w}_\mu =  0$ trivially holds.
  We then prove $w^\top A w = 0 \implies v_w = 0$ a.e. via contradiction.
  First, we have $\indot{v_w}{\gamma P_\pi v_w - v_w}_\mu = 0$ because $w^\top A w = 0$.
  Now, suppose $\norm{v_w}_\mu > 0$.
  We have
  \begin{align}
    \textstyle \indot{v_w}{\gamma P_\pi v_w - v_w}_\mu
    =& \textstyle \indot{v_w}{\gamma P_\pi v_w}_\mu - \norm{v_w}^2_\mu
    \le \textstyle \gamma \norm{v_w}_\mu \norm{P_\pi v_w}_\mu - \norm{v_w}^2_\mu\\
    \le& \textstyle (\gamma - 1)\norm{v_w}^2_\mu \explain{Lemma~\ref{lemma: pv norm}}\\
    <& \textstyle 0 \explain{$\norm{v_w}_\mu > 0$}.
  \end{align}
  This generates a contradiction and completes the proof.
  Thus, we must have $\norm{v_w}_\mu = 0$, which holds if and only if $v_w = 0$ a.e.
\end{proof}

\subsection{Proof of Lemma~\ref{lemma: W nonempty}}
\label{proof: W nonempty}
\begin{proof}
Recall that we can represent the vector-valued feature mapping $x$ as
$x(s) = \mqty[x_1(s) & x_2(s) & \cdots & x_d(s)]^\top$, where each $x_i: \fS \to \R$ is a basis function 
for $i = 1,2,\dots,d$.
Without further assumptions on $x$, the basis functions can be linearly dependent.
Without loss of generality, suppose the first $m$ basis functions form the largest linearly independent collection of $x$.
When $m = 0$, it implies $x = 0$ a.e., and we end up with a degenerate case, where $A = 0$ and $b = 0$.
In this case, $\fW_* = \R[d]$ and clearly is non-empty.
In the rest of the proof, we analyze the case where $m \ge 1$.
We denote $\phi(s) \doteq \mqty[x_1(s) & x_2(s) & \cdots & x_m(s)]^\top \in \R[m]$.
Then, for each $x_k$, where $k \in \qty{m+1, \dots, d}$, there exists a vector of coefficients $c \doteq \mqty[c_1 & c_2 \cdots & c_m]$, such that $x_k = \sum_{i=1}^m c_i x_i$ a.e. 
Therefore, we can define $\hat{x}(s) \doteq \mqty[\phi(s) \\ C \phi(s)]$, where $C \in \R[(d-m) \times m]$ is a coefficient matrix, such that $\hat{x} = x$ a.e.
Define
$\hat{A} \doteq \int_{\fS} \hat{x}(s)\qty(\gamma (P_\pi \hat{x})(s)^\top - \hat{x}(s)^\top) \mu(\dd{s})$
and
$\hat{b} \doteq \int_{\fS} \hat{x}(s)r(s)\mu(\dd{s})$.
We have
\begin{align}
  \hat{A} - A
  =&\textstyle\int_{\fS}(\hat{x}(s) - x(s) + x(s))\qty(\gamma(P_\pi \hat{x})(s)^\top - \hat{x}(s)^\top) \mu(\dd{s}) - A\\
  =&\textstyle\int_{\fS}(\hat{x}(s) - x(s))\qty(\gamma(P_\pi \hat{x})(s)^\top - \hat{x}(s)^\top) \mu(\dd{s})\\
  +& x(s)\qty(\gamma(P_\pi \hat{x})(s)^\top - \hat{x}(s)^\top) \mu(\dd{s}) - A \\
  =&\textstyle\int_\fS \gamma x(s)(P_\pi \hat{x})(s)^\top \mu(\dd{s}) - \int_\fS x(s)\hat{x}(s)^\top \mu(\dd{s}) - A \explain{$\hat x = x$ a.e.} \\
  =&\textstyle\int_\fS \gamma x(s)(P_\pi \hat{x})(s)^\top \mu(\dd{s}) - \int_\fS x(s)\hat{x}(s)^\top \mu(\dd{s})\\
  -& \textstyle\qty(\int_\fS \gamma x(s)(P_\pi x)(s)^\top \mu(\dd{s}) - \int_\fS x(s)x(s)^\top \mu(\dd{s}))\\
  =&\textstyle\int_\fS \gamma x(s)(P_\pi (\hat{x} - x))(s)^\top \mu(\dd{s}) - \int_\fS x(s)(\hat{x} - x)(s)^\top \mu(\dd{s})\\
  =&\textstyle\int_\fS \gamma x(s)(P_\pi (\hat{x} - x))(s)^\top \mu(\dd{s}) \explain{$\hat{x} - x = 0$ a.e.}
\end{align}
By Assumption~\ref{assumption: transition kernel map zero to zero}, it holds that $P_\pi(\hat{x} - x) = 0$ a.e.
Consequently, we have $\hat{A} - A = 0$.
Similarly, we have
$\hat{b} - b = \int_{\fS} (\hat{x} - x)(s)r(s)\mu(\dd{s}) = 0$.
Thus, it holds that $\hat{A} = A$ and $\hat{b} = b$.

We now prove that $\fW_*$ is non-empty.
Let $w = \mqty[w_1 \\ w_2]$, where $w_1 \in \R[m]$ and $w_2 \in \R[d-m]$.
We have
$Aw + b =
  \int_\fS \mqty[\phi(s) \\ C \phi(s)]
  \mqty[\gamma(P_\pi \phi)(s) - \phi(s)\\
        C \qty(\gamma(P_\pi \phi)(s) - \phi(s))]^\top \mu(\dd{s})
  \mqty[w_1 \\ w_2]
  + 
  \int_{\fS} \mqty[\phi(s) \\ C \phi(s)] r_\pi(s)\mu(\dd{s})$.
Define $\delta(s) \doteq \gamma(P_\pi \phi)(s) - \phi(s)$.
We then have
\begin{align}
  A w + b
  =&
  \mqty[\int_\fS \phi(s)\delta(s)^\top\mu(\dd{s}) &
        \int_\fS \phi(s)\delta(s)^\top\mu(\dd{s}) C^\top\\
        C \int_\fS  \phi(s)\delta(s)^\top \mu(\dd{s}) &
        C \int_\fS  \phi(s)\delta(s)^\top \mu(\dd{s}) C^\top]
  \mqty[w_1 \\ w_2]
  + 
  \mqty[\int_\fS \phi(s)r_\pi(s)\mu(\dd{s})\\
        C \int_\fS \phi(s)r_\pi(s)\mu(\dd{s})]\\
  =&
  \mqty[A' & A' C^\top\\ CA' & CA'C^\top]
  \mqty[w_1 \\ w_2]
  + 
  \mqty[b' \\ C b'],
\end{align}
where we use shorthands $A' \doteq \int_\fS \phi(s)\delta(s)^\top\mu(\dd{s})$
and $b' \doteq \int_\fS \phi(s)r_\pi(s)\mu(\dd{s})$.
Thus, it leaves us with two linear equations
\begin{align}
  \label{eq: full linear equation}
  \begin{cases}
    A'w_1 + A'C^\top w_2 + b'&= 0\\
    CA'w_1 + CA'C^\top w_2 + C b'&= 0
  \end{cases}
\end{align}
which we can solve to obtain some $w_1, w_2$, such that $A\mqty[w_1\\w_2] + b = 0$.
Let $\ker(M)$ denote the kernel space of a matrix $M$.
We choose $w_2 \in \ker(A'C^\top)$, such that~\eqref{eq: full linear equation} reduces to
\begin{align}
  \label{eq: reduced linear equation}
  \begin{cases}
    A' w_1 &= -b'\\
    CA'w_1 &= -Cb'.
  \end{cases}
\end{align}
We then observe that $w_1 = -A'^{-1}b'$ is a valid solution to the first equation of~\eqref{eq: reduced linear equation} provided $A'$ has full rank.
Furthermore, this $w_1$ would also satisfy the second equation of~\eqref{eq: reduced linear equation}, thus solving the system of linear equations.
We now prove that $A'$ is indeed invertible by showing it is negative definite.
The proof mirrors the proof of Lemma~\ref{lemma: A NSD}.
We first define $v_w'(s) \doteq \phi(s)^\top w$ for all $s \in \fS$ and any $w \in \R[m]\setminus\qty{0}$, which is square-integrable under Assumption~\ref{assumption: feature function bounded}.
Recall that the basis functions $x_1, x_2, \dots, x_m$ that constitute $\phi$ are linearly independent.
Hence, it is impossible that $v'_w = 0$ a.e.
We then have
\begin{align}
    \indot{w}{A'w} 
    =& \textstyle\indot{v'_w}{\gamma P_\pi v'_w - v'_w}_\mu
    = \textstyle\indot{v'_w}{\gamma P_\pi v'_w}_\mu - \norm{v'_w}^2_\mu
    \le \textstyle\gamma \norm{v'_w}_\mu \norm{P_\pi v'_w}_\mu - \norm{v'_w}^2_\mu\\
    \le& \textstyle(\gamma - 1)\norm{v'_w}^2_\mu \explain{Lemma~\ref{lemma: pv norm}}
    < \textstyle0.
\end{align}
The matrix $A'$ is negative definite and thus invertible.
As a result, we have constructed a solution set 
$\qty{w = \mqty[w_1 \\ w_2] \Bigg| 
  w_1 = - A'^{-1} b'; w_2 \in \ker(A'C^\top)} 
  \subseteq \fW_*$
and proved $\fW_*$ is non-empty when $m \ge 1$.
\end{proof}

\subsection{Proof of Lemma~\ref{lemma: v = v' a.e. iff w' in W}}
\label{proof: v = v' a.e. iff w' in W}
\begin{proof}
  We first recall that $v_{w - w'} \doteq x(s)^\top (w - w')$.
  We begin with $w' \in \fW_* \implies v_w = v_{w'}$ a.e.
  Suppose that $w' \in \fW_*$.
  We have
  \begin{align}
    \textstyle Aw + b - (Aw' + b) =& 0\\
    \textstyle A(w - w') =& 0\\
    \textstyle (w - w')^\top A (w - w') =& 0.
  \end{align} 
  By Lemma~\ref{lemma: inner integral positive definite},  we have $v_{w - w'} = 0$ a.e., implying $v_w = v_{w'}$ a.e.

  We now show $v_w = v_{w'}\quad\text{a.e.} \implies w'\in\fW_*$.
  Suppose $v_w = v_{w'}$ a.e. holds.
  We then have $v_{w - w'} = 0$ a.e.
  Define $\abs{v_{w - w'}}(s) \doteq \abs{x(s)^\top(w - w')}$.
  Then, it holds trivially that $\abs{v_{w - w'}} = 0$ a.e.
  We have
  \begin{align}
    \textstyle \norm{A(w - w')}
    =& \textstyle \norm{\int_{\fS} x(s)\qty(\gamma (P_\pi x)(s)^\top - x(s)^\top) \mu(\dd{s}) (w - w')}\\
    =& \textstyle \norm{\int_{\fS} x(s)\qty(\gamma (P_\pi v_{w - w'})(s)- v_{w - w'}(s)) \mu(\dd{s})}\\
    \le& \textstyle \int_{\fS} \norm{x(s)\qty(\gamma (P_\pi v_{w - w'})(s)- v_{w - w'}(s))} \mu(\dd{s}) \explain{Jensen's inequality}\\
    =& \textstyle \int_{\fS} \norm{x(s)}\abs{\qty(\gamma (P_\pi v_{w - w'})(s)- v_{w - w'}(s))} \mu(\dd{s}) \\
    \le& \textstyle \int_{\fS} \norm{x(s)}\gamma \abs{(P_\pi v_{w - w'})(s)}\mu(\dd{s})
    + \int_{\fS} \norm{x(s)} \abs{v_{w - w'}}(s) \mu(\dd{s})\\
    \le& \textstyle C_x\int_{\fS}\gamma \abs{(P_\pi v_{w - w'})(s)}\mu(\dd{s})
    + C_x \int_{\fS} \abs{v_{w - w'}}(s) \mu(\dd{s}) 
    \explain{$C_x \doteq \sup_{s\in\fS} \norm{x(s)} < \infty$ by Assumption~\ref{assumption: feature function bounded}}\\
    =& \textstyle C_x\int_{\fS}\gamma \abs{(P_\pi v_{w - w'})(s)}\mu(\dd{s}) \explain{$\abs{v_{w - w'}} = 0$ a.e. w.r.t. $\mu$ }\\
    =& 0. \explain{Assumption~\ref{assumption: transition kernel map zero to zero}}
  \end{align}
  The result suggests $A(w - w') = 0$, which implies $Aw = Aw'$.
  We therefore have $Aw' + b = Aw + b = 0$ and have proved that $w' \in \fW_*$.
\end{proof}

\section{Proofs in Section~\ref{section: ode analysis}}  
\begin{xlemma}
  \label{lemma: w(t) bounded}
  Under Assumptions~\ref{assumption: ergodicity} and~\ref{assumption: feature function bounded},
  it holds that $\sup_{t \in [0, \infty)} \norm{w(t; w_0)} < \infty$
  for all $w_0 \in \R[d]$.
\end{xlemma}
\begin{proof}
  Fix an arbitrary $w_* \in \fW_*$.
  We first show that $\dv{\norm{w(t; w_0) - w_*}^2}{t} \le 0$ for any $w_0 \in \R[d]$.
  \begin{align}
    \textstyle \dv{\norm{w(t; w_0) - w_*}^2}{t}
    &= \textstyle 2(w(t; w_0) - w_*)^\top (A w(t; w_0) + b)\\
    &= \textstyle 2(w(t; w_0) - w_*)^\top (A w(t; w_0) + b - (A w_* + b)) \explain{$A w_* + b = 0$}\\
    &= \textstyle 2(w(t; w_0) - w_*)^\top A (w(t; w_0) - w_*)
    \le \textstyle 0. \explain{Lemma~\ref{lemma: A NSD}}
  \end{align}
  Thus, $\norm{w(t; w_0) - w_*}$ is monotonically decreasing given any $w_0 \in \R[d]$.
  Hence, \\$\norm{w(t; w_0) - w_*} \le \norm{w_0 - w_*}$ for $t \ge 0$, and by triangle inequality we get\\
  $\sup_{t \in [0, \infty)}\norm{w(t; w_0)} \le \norm{w_*} + \norm{w_0 - w_*}$,
  which completes the proof.
\end{proof}

\subsection{Proof of Lemma~\ref{lemma: convergence re(lambda) negative}}
\label{proof: convergence re(lambda) negative}
\begin{proof}
  The absolute homogeneity of a matrix norm implies that
  \begin{align}
      &\textstyle\lim_{t \to \infty} \norm{\exp(\Re(\lambda_i) t) \qty(\cos(\Im(\lambda_i) t) + \mi \sin(\Im(\lambda_i) t))
      \sum_{n=0}^{y_{i,j}-1} \frac{1}{n!} t^n N_{i,j}^n}\\
      =& \textstyle\lim_{t \to \infty} \exp(\Re(\lambda_i) t) \abs{\cos(\Im(\lambda_i) t) + \mi \sin(\Im(\lambda_i) t)}
      \norm{\sum_{n=0}^{y_{i,j}-1} \frac{1}{n!} t^n N_{i,j}^n}\\
      =& \textstyle\lim_{t \to \infty} \exp(\Re(\lambda_i) t)  \norm{\sum_{n=0}^{y_{i,j}-1} \frac{1}{n!} t^n N_{i,j}^n}.
  \end{align}
  It appears that
  $\norm{\sum_{n=0}^{y_{i,j}-1} \frac{1}{n!} t^n N_{i,j}^n} = \mathcal{O} (t^{y_{i,j} - 1})$.
  When $\Re(\lambda_i) < 0$, we have
  $\lim_{t \to \infty} \\ \exp(\Re(\lambda_i) t)
  \norm{\sum_{n=0}^{y_{i,j}-1} \frac{1}{n!} t^n N_{i,j}^n} = 0$
  because $\exp(\Re(\lambda_i) t)$ decays exponentially,
  whereas $\norm{\sum_{n=0}^{y_{i,j}-1} \frac{1}{n!} t^n N_{i,j}^n}$ exhibits polynomial growth.
\end{proof}

\subsection{Proof of Lemma~\ref{lemma: exp^Bt bounded}}
\label{proof of e^Bt bounded}
\begin{proof}
  We shall prove this claim by propagating the boundedness of $w(t; w_0)$ through different levels, starting from $z(t; z_0)$.
  For any $z_0 \in \R[d], w_* \in \fW_*$, pick $w_0 = z_0 + w_*$. 
  From~\eqref{eq z eq w}, we obtain
  \begin{align}
      \textstyle\sup_{t \in [0, \infty)} \norm{z(t; z_0)}
      &= \textstyle\sup_{t \in [0, \infty)} \norm{w(t; z_0 + w_*) - w_*}\\
      &\le \textstyle\sup_{t \in [0, \infty)} \norm{w(t; z_0 + w_*)} + \norm{w_*}
      < \textstyle\infty \explain{Lemma~\ref{lemma: w(t) bounded}}.
  \end{align}
  Then, we prove that $\exp(At)$ is bounded.
  Let $e_n$ be the $n$-th column of $I_d$, 
  the identity matrix in $\R[d\times d]$.
  By the boundedness of $z(t; z_0)$ and~\eqref{eqn: closed form solution}, 
  we have for all $n \in \qty{1,2,\dots,d}$, $\sup_{t \in [0,\infty)} \norm{\exp(At) e_n} = \sup_{t \in [0, \infty)} \norm{z(t; e_n)} < \infty$.
  By the definition of the induced norm and the invariance of $\ell_2$ norm under matrix transpose, we obtain $\norm{\exp(At)} \leq \sqrt{d} \max_{n} \norm{\exp(At)e_n} = \sqrt{d} \max_n \norm{z(t;e_n)}$.
  Since $n$ is finite, we have
  $\sup_{t \in [0,\infty)} \norm{\exp(At) I_d} \\\leq \sup_{t\in [0, \infty)}\sqrt{d} \max_n \norm{z(t;e_n)} < \infty$.
  It then follows from the Jordan decomposition~\eqref{eqn: exp(At) Jordan normal form} that
  $\sup_{t \in [0, \infty)} \norm{\exp(Jt)} = \sup_{t \in [0, \infty)} \norm{P^{-1} \exp(At) P} \le \sup_{t \in [0, \infty)} \norm{\exp(At)} \norm{P^{-1}} \norm{P} < \infty$.
  Since $\exp(Jt)$ is a block diagonal matrix,
  we have
  $\max_{i,j} \norm{\exp(B_{i,j}t)} = \norm{\exp(Jt)} < \infty$,
  which completes the proof.
\end{proof}
\subsection{Proof of Corollary~\ref{corollary: dz/dt diminish}}
\label{proof: dz/dt diminish}
\begin{proof}
For any $z_0 \in \R[d], w_* \in \fW_*$, pick $w_0 = z_0 + w_*$.
From~\eqref{eq z eq w}, we obtain
$z(t; z_0) = w(t; z_0 + w_*) - w_*$.
By Theorem~\ref{thm: ode value convergence}, we have $\lim_{t\to\infty} v_{w(t; z_0 + w_*)} = v_{w_*}$ a.e.
Thus, 
we have $\lim_{t\to\infty} v_{z(t;z_0)} = \lim_{t\to\infty} v_{w(t; z_0 + w_*)} - v_{w_*} = 0$ a.e. 
Then, according to Corollary~\ref{corollary: vz equals 0 a.e. iff z in Z}, it holds that $\lim_{t\to\infty} d(z(t; z_0), \fZ_*) = 0$.
Since $\dv{z(t; z_0)}{t} = Az(t; z_0)$, we have $\lim_{t\to\infty}\dv{z(t; z_0)}{t} = 0$ by the definition of $\fZ_*$ and the fact that $z(t;z_0)$ converges to $\fZ_*$.
\end{proof}

\subsection{Proof of Lemma~\ref{lemma: convergence re(lambda) zero}}
\label{proof: convergence re(lambda) zero}
\begin{proof}
  Given $\Re(\lambda_i) = 0$, we have\\
  $\exp(B_{i,j} t) = \qty(\cos\qty(\Im(\lambda_i) t) + \mi \sin\qty(\Im(\lambda_i) t)) \sum_{n=0}^{y_{i,j}-1} \frac{1}{n!} t^n N_{i,j}^n$.
  We first prove $y_{i, j} = 1$ by contradiction.
  Suppose $y_{i,j} > 1$.
  If $\Im(\lambda_i) = 0$, 
  we then have
  \begin{align}
    \textstyle \lim_{t\to\infty}\norm{\exp(B_{i,j} t)} = \lim_{t\to\infty}\norm{\sum_{n=0}^{y_{i,j}-1} \frac{1}{n!} t^n N_{i,j}^n} = \infty,
  \end{align}
  yielding a contradiction with Lemma~\ref{lemma: exp^Bt bounded}.
  If $\Im(\lambda_i) \neq 0$, then $\exp(B_{i,j} t)$ oscillates with $t$.
  Consider the sequence $\qty{t_k: t_k = \frac{2k\pi}{\abs{\Im(\lambda_i)}}, k \ge 0}$.
  We have 
  \begin{align}
    \textstyle \lim_{k\to\infty}\norm{\exp(B_{i,j} t_k)} = \lim_{k\to\infty}\norm{\sum_{n=0}^{y_{i,j}-1} \frac{1}{n!} t_k^n N_{i,j}^n} = \infty,
  \end{align}
  again leading to a contradiction with Lemma~\ref{lemma: exp^Bt bounded}.
  We then conclude that $y_{i, j} = 1$ must hold, and 
  we are left with
  $\exp(B_{i,j} t) =\qty(\cos\qty(\Im(\lambda_i) t) + \mi \sin\qty(\Im(\lambda_i) t)) I_{\rho_{i,j}}$.
  We now prove $\Im(\lambda_i) = 0$ by contradiction.
  Assume that $\Im(\lambda_i) \neq 0$ holds.
  Then,
  \begin{align}
    \textstyle
    \label{eqn: d/dt goes to 0}
    \dv{\exp(B_{i,j}t)}{t} = \Im(\lambda_i) \qty(\mi \cos(\Im(\lambda_i) t) - \sin(\Im(\lambda_i) t)) I_{\rho_{i,j}}
  \end{align}
  and $\dv{\exp(B_{i,j}t)}{t}$ would oscillate and never converge.
  On the other hand, we have
  $\dv{z(t;z_0)}{t} = \dv{}{t} \exp(At) z_0 = P \dv{}{t}\exp(Jt)P^{-1} z_0$.
  We have showed in Corollary~\ref{corollary: dz/dt diminish} that $\lim_{t\to\infty} \dv{z(t;z_0)}{t} = 0$ for all $z_0 \in \R[d]$.
  Hence, by setting $z_0 = e_1, e_2, \dots, e_d$, where $e_n$ is the $n$-th column of $I_d$, it holds that
  \begin{align}
    \textstyle \lim_{t\to\infty} P \dv{}{t}\exp(Jt)P^{-1} I_d &= 0\\
    \textstyle P^{-1}\lim_{t\to\infty} P \dv{}{t}\exp(Jt)P^{-1} P &= 0\\
    \textstyle \lim_{t\to\infty} \dv{}{t}\exp(Jt) &= 0,
  \end{align}
  implying that $\lim_{t\to\infty} \dv{}{t} \exp(B_{i,j} t) = 0$,
  yielding a contradiction with~\eqref{eqn: d/dt goes to 0}.
  We then conclude that $\Im(\lambda_i) = 0$.
  We are finally left with $\forall t \ge 0$,
  $\exp(B_{i,j}t) = I_{\rho_{i,j}}$,
  which completes the proof.
\end{proof}

\subsection{Proof of Corollary~\ref{corollary: bounded invariant set}}
\label{proof: bounded invariant set}
\begin{proof}
  Let $\fW$ be a bounded invariant set of ODE~\eqref{eq linear td ode}.
  By definition, for every $w_0 \in \fW$, the solution $w(t; w_0)$ to ODE~\eqref{eq linear td ode} on the domain $(-\infty, \infty)$ would remain in $\fW$.
  Furthermore, since $\fW$ is also bounded, it implies that $w(t; w_0)$ is a bounded solution on $(-\infty, \infty)$.
  By Theorem~\ref{theorem: bounded solution constant}, it holds that $w(t;w_0)$ is constant and in $\fW_*$.
  Thus, it implies that $w_0\in \fW_*$,
  which completes the proof.
\end{proof}

\section{Proofs in Section~\ref{section: stochastic approximation}}
\label{sec sa proof}
We recall for analyzing~\eqref{eqn: linear td update},
we have
$H(w, y) \doteq \qty(r(s) + \gamma w^\top x(s') - w^\top x(s))x(s),\\
  h(w) \doteq \E_{y\sim\eta}\qty[H(w,y)]$.
Some elementary properties then follow.

\begin{xlemma}
  \label{lemma: H Lipschitz continuous}
  Let Assumption~\ref{assumption: feature function bounded} hold.
  Then there exists some $K_\tref{lemma: H Lipschitz continuous} > 0$, such that for any $w, w'\in\R[d]$ and $y \in \fY$, it holds that $\norm{H(w, y) - H(w',y)} \le K_\tref{lemma: H Lipschitz continuous}\norm{w - w'}$.
\end{xlemma}
\begin{proof}
  Given $w, w' \in \R[d]$, we have
  \begin{align}
    \textstyle\norm{H(w, y) - H(w', y)}
    =& \textstyle \norm{\qty(\gamma(w - w')^\top x(s') - (w - w')^\top x(s))x(s)}
    \explain{$Y \doteq \mqty[s & s']^\top$}\\
    =& \textstyle \abs{\gamma\indot{w - w'}{x(s')} - \indot{w - w'}{x(s)}}\norm{x(s)}\\
    \le& \textstyle C_x\qty(\abs{\gamma\indot{w - w'}{x(s')}} + \abs{\indot{w - w'}{x(s)}})
    \explain{$C_x \doteq \sup_{s\in\fS} \norm{x(s)} < \infty$}\\
    \le& \textstyle C_x\qty(\gamma C_x \norm{w-w'} + C_x\norm{w-w'})\\
    =& \textstyle (1 + \gamma)C_x^2\norm{w - w'}.
  \end{align}
  Setting $K_\tref{lemma: H Lipschitz continuous} = (1 + \gamma)C_x^2$, the lemma is proved.
\end{proof}

\begin{xlemma}
  \label{lemma: h Lipschitz continuous}
  Let Assumptions~\ref{assumption: ergodicity} and~\ref{assumption: feature function bounded} hold.
  Then there exists some $K_\tref{lemma: h Lipschitz continuous} > 0$, such that for any $w, w'\in\R[d]$, it holds that $\norm{h(w) - h(w')} \le K_\tref{lemma: h Lipschitz continuous}\norm{w - w'}$.
\end{xlemma}
\begin{proof}
  \begin{align}
    \textstyle \norm{h(w) - h(w')}
    =& \textstyle \norm{\E\textstyle _{Y \sim \eta}\qty[H(w, Y) - H(w', Y)]}\\
    \le& \textstyle \E_{Y \sim \eta}\qty[\norm{H(w, Y) - H(w', Y)}]
    \le \textstyle \E_{Y \sim \eta}\qty[K_\tref{lemma: H Lipschitz continuous}\norm{w-w'}]
    = \textstyle K_\tref{lemma: H Lipschitz continuous}\norm{w-w'}.
  \end{align}
  Hence, the inequality holds by setting $K_\tref{lemma: h Lipschitz continuous} = K_\tref{lemma: H Lipschitz continuous}$.
\end{proof}

\begin{xlemma}
  \label{lemma: H linear growth}
  Let Assumption~\ref{assumption: feature function bounded} hold.
  Then there exists a constant $K_\tref{lemma: H linear growth} > 0$, such that for all $y \in \fY$, it holds that $\norm{H(w, y)} \le K_\tref{lemma: H linear growth}(1 + \norm{w})$.
\end{xlemma}
\begin{proof}
  We have by Lemma~\ref{lemma: H Lipschitz continuous}
  $\norm{H(w, y) - H(w', y)} \le K_\tref{lemma: H Lipschitz continuous}\norm{w-w'}$.
  Then, fixing an arbitrary $\tilde{w} \in \R[d]$ with $\norm{\tilde{w}} = 1$, we have
  \begin{align}
    \textstyle \norm{H(w, y) - H(\tilde{w}, y)} 
    \le& \textstyle K_\tref{lemma: H Lipschitz continuous}\norm{w - \tilde{w}}\\
    \textstyle \norm{H(w, y)} - \norm{H(\tilde{w}, y)} 
    \le& \textstyle K_\tref{lemma: H Lipschitz continuous}\norm{w - \tilde{w}}\\
    \textstyle \norm{H(w, y)} 
    \le& \textstyle K_\tref{lemma: H Lipschitz continuous}\norm{w - \tilde{w}} + \norm{\qty(r(s) + \gamma \tilde{w}^\top x(s') - \tilde{w}^\top x(s))x(s)}\\
    \le& \textstyle K_\tref{lemma: H Lipschitz continuous}\qty(\norm{w} + \norm{\tilde{w}})
    + C_x \abs{r(s) + \gamma \tilde{w}^\top x(s') - \tilde{w}^\top x(s)}
    \explain{$C_x \doteq \sup_{s\in\fS} \norm{x(s)} < \infty$}\\
    \le& \textstyle K_\tref{lemma: H Lipschitz continuous}(\norm{w} + 1)
    C_x\qty(\abs{r(s)} + \gamma \abs{\tilde{w}^\top x(s')} + \abs{\tilde{w}^\top x(s)})\\
    \le& \textstyle K_\tref{lemma: H Lipschitz continuous}(\norm{w} + 1)
    C_x\qty(C_r + \gamma \norm{\tilde{w}}\norm{x(s')} + \norm{\tilde{w}} \norm{x(s)})
    \explain{$C_r \doteq \sup_{s\in\fS}\abs{r(s)} < \infty$}\\
    \le& \textstyle K_\tref{lemma: H Lipschitz continuous}(\norm{w} + 1)
    C_x\qty(C_r + \gamma C_x + C_x)\\
    =& \textstyle K_\tref{lemma: H Lipschitz continuous}\qty(C_x C_r + \gamma C_x^2 + C_x^2)
    (\norm{w} + 1).
  \end{align}
  Hence, setting $K_\tref{lemma: H linear growth} \doteq K_\tref{lemma: H Lipschitz continuous}\qty(C_x C_r + \gamma C_x^2 + C_x^2)$, we get $\norm{H(w,y)} \le K_\tref{lemma: H linear growth} (1 + \norm{w})$, which completes the proof.
\end{proof}

\subsection{Proof of Theorem~\ref{theorem: w_t a.s. bounded}}
\label{proof: w_t a.s. bounded}
\begin{proof}
  We shall show that Assumptions~\ref{assumption: learning rate sa} and~\ref{assumption: h bounded} -- \ref{assumption: bounded moment} hold.
  Invoking Theorem~\ref{theorem: w_t a.s. bounded conditions} will then complete the proof.
  First, Assumption~\ref{assumption: learning rate sa} is satisfied by Assumption~\ref{assumption: learning rate}.
  Assumption~\ref{assumption: h bounded} holds by Lemma~\ref{lemma: H linear growth}.
  We then proceed to show Assumption~\ref{assumption: g and nu exist} can be satisfied.
  Choosing $g = h$, where we recall that $h(w) \doteq \int_\fY H(w, y) \eta(\dd{y})$,
  Assumption~\ref{assumption: g and nu exist} (i) holds by Lemma~\ref{lemma: h Lipschitz continuous}.
  Assumption~\ref{assumption poisson} automatically satisfies $(ii)$, $(iii)$, and $(iv)$ of Assumption~\ref{assumption: g and nu exist}.
  Lastly, we prove Assumption~\ref{assumption: bounded moment} holds.
  Since $\fS$ is compact by Assumption~\ref{assumption: state space compact}, then $\fY$ is also compact.
  Thus, $\sup_{y\in\fY} \norm{y}$ is finite.
  We denote this quantity as $C_y$.
  Then, for all $a \in \R[d]$, $q > 0$ and $n \ge 0$, it holds that
  $\E\qty[1 + \norm{Y_{n+1}}^q \mid w_0 = a, Y_0 = y]
    \le 1 + C_y^q$.
  Thus, by setting $K_\tref{assumption: bounded moment}(q) \doteq (1 + C_y^q)$, Assumption~\ref{assumption: bounded moment} is satisfied.

  Having verified the assumptions, we now show the existence of the function $U$ required by Theorem~\ref{theorem: w_t a.s. bounded conditions}.
  Fixing an arbitrary $w_* \in \fW_*$, we define $U$ as
  $U(w) \doteq \norm{w - w_*}^2 + \norm{w_*}^2$.
  The second derivative of $U$ is $2 I_d$.
  Therefore, $U$ is of class $C^2$ with bounded second derivatives.
  We now show that it satisfies Condition (i) of Theorem~\ref{theorem: w_t a.s. bounded conditions}.
  Taking the gradient, we have
  \begin{align}
    \textstyle \innerdot{\nabla_w U(w)}{h(w)} 
    = \textstyle \innerdot{2(w - w_*)}{Aw + b}
    &= \textstyle 2\innerdot{(w - w_*)}{Aw + b - \qty(Aw_* + b)}\\
    &= \textstyle 2\innerdot{w - w_*}{A(w - w_*)}
    \le \textstyle 0 \explain{Lemma~\ref{lemma: A NSD}}.
  \end{align}
  Condition (i) holds for our selected $U$.
  Next, we prove that Condition (ii) also holds by showing that $U(w) \ge \frac{1}{2}\norm{w}^2$ for all $w \in \R[d]$.
  We have
  \begin{align}
    \textstyle U(w) - \frac{1}{2} \norm{w}^2 
    &= \textstyle \norm{w - w_*}^2 + \norm{w_*}^2 - \frac{1}{2}\norm{w}^2
    = \textstyle \frac{1}{2}\norm{w}^2 - 2\innerdot{w}{w_*} + 2 \norm{w_*}^2\\
    &= \textstyle \frac{1}{2}\qty(\norm{w}^2 - 4\innerdot{w}{w_*} + 4\norm{w_*}^2)
    = \textstyle \frac{1}{2} \norm{w - 2w_*}^2
    \ge \textstyle 0.
  \end{align}
  Hence, Condition (ii) holds by setting $c = \frac{1}{2}$.
  The proof is completed.
\end{proof}

\subsection{Proof of Theorem~\ref{thm: linear TD converge to invariant set}}
\label{proof: linear TD converge to invariant set}
\begin{proof}
  We apply Theorem~\ref{thm: shuze converge to bounded invariant set} to show that
  $\qty{w_t}$ converges almost surely to a possibly sample path dependent bounded invariant set of ODE~\eqref{eq linear td ode}.
  To be able to do so, we need to verify Assumptions \ref{assumption: learning rate sa}, \ref{assumption: Y_t ergodic}, \ref{assumption: L(y) existence} and \ref{assumption: w_t bounded} hold.
  Assumptions A.1 and A.2 are satisfied by Assumptions~\ref{assumption: learning rate} and~\ref{assumption: ergodicity}, respectively.
  Regarding Assumption~\ref{assumption: L(y) existence}, we have
  \begin{align}
    \textstyle \norm{H(w, y) - H(w', y)} 
    &=\textstyle \norm{\qty(\gamma \indot{x(s')}{w - w'} - \indot{x(s)}{w - w'})x(s)}\\
    &=\textstyle \abs{\indot{\gamma x(s') - x(s)}{w - w'}}\norm{x(s)}\\
    &\le \textstyle C_x \abs{\indot{\gamma x(s') - x(s)}{w - w'}} \explain{$C_x \doteq \sup_{s\in\fS} \norm{x(s)} < \infty$}\\
    &\le \textstyle C_x \norm{w - w'} \norm{\gamma x(s') - x(s)}\\
    &\le \textstyle C_x \norm{w - w'} \qty(\gamma \norm{x(s')} + \norm{x(s)})
    \le \textstyle C_x \norm{w - w'} \qty(C_x + \gamma C_x)\\
    &= \textstyle (1 + \gamma) C_x^2 \norm{w - w'}.
  \end{align}
  Thus, simply setting $L(y) = (1 + \gamma)C_x^2$ satisfies Assumption~\ref{assumption: L(y) existence}. 
  Lastly, Assumption~\ref{assumption: w_t bounded} is satisfied by Theorem~\ref{theorem: w_t a.s. bounded}.
  As a result, we have $\qty{w_t}$ converging almost surely to a bounded invariant set of ODE \eqref{eq linear td ode}.
  In light of Corollary~\ref{corollary: bounded invariant set}, any bounded invariant set is a subset of $\fW_*$.
  Hence, $\qty{w_t}$ converges almost surely to $\fW_*$.
  Lemma~\ref{lemma: v = v' a.e. iff w' in W} then completes the proof of Theorem~\ref{thm: linear TD converge to invariant set}.
\end{proof}
\subsection{Proof of Corollary~\ref{corollary: local stability}}
\label{proof: local stability}
\begin{proof}
  For any sample path $\qty{w_0, w_1, \dots}$,
  let $\bar{w}: \R \to \R[d]$
  be the piece-wise constant interpolation of $\qty{w_t}$,
  i.e.,
  \begin{align}
    \bar{w}(t) 
    \doteq
    \begin{cases}
        w_{m(0, t)} & t > 0\\
        0  &t \le 0.
    \end{cases}
  \end{align}
  Intuitively, on the positive real axis, $\bar{w}(t)$ is the piece-wise constant interpolation of $\qty{w_t}$,
  with each piece having a length $\qty{\alpha_t}$.
  We then define a helper function $\tau: \mathbb{N} \to \R$ as
  $\tau(k) \doteq \sum_{i=0}^k \alpha_k$.
  Next, we define a sequence of functions $\qty{f_i: \R \to \R[d]}$, where
  $f_i(t) \doteq \bar{w}\qty(\tau(i) + t)$.
  The boundedness of $\qty{w_t}$ follows from Theorem~\ref{theorem: w_t a.s. bounded}, implying that $\bar{w}(t)$ is also bounded, which in turn ensures that $\qty{f_i}$ is a sequence of bounded functions.
  Under Assumptions \ref{assumption: learning rate sa}, \ref{assumption: Y_t ergodic}, \ref{assumption: L(y) existence}, \ref{assumption: w_t bounded} (these assumptions have been verified in the proof of Theorem~\ref{thm: linear TD converge to invariant set}), and \ref{assumption lil}, Lemma~32 and the conclusion right above Lemma~34 of \cite{liu2025ode} proves there exists at least one convergent subsequence $\qty{f_{i_k}} \subseteq \qty{f_i}$, such that
  $\lim_{k\to\infty} f_{i_k} (t)
    = \lim_{k\to\infty} \bar{w}\qty(\tau(i_k) + t)
    = \hat{w}(t)$,
  where $\hat{w}(t)$ is a bounded solution to ODE~\eqref{eq linear td ode} on $\qty(-\infty, \infty)$.
  Additionally, Theorem~\ref{theorem: bounded solution constant} guarantees that $\hat{w}(t)$ is a constant solution and in $\fW_*$.
  Hence, we can conclude that for all $T < \infty$, we have
  $\lim_{k \to \infty} \bar{w}(\tau(i_k) + T) = w_*$, 
  where $w_* \in \mathcal{W}_*$.
  Since $\bar{w}(t)$ is merely a piece-wise constant interpolation of the weight sequence $\qty{w_t}$, the statement of this corollary immediately holds.
\end{proof}

\section{Proofs in Section~\ref{section: finite state space}}
\subsection{Proof of Lemma~\ref{lemma: projection matrix}}
\label{proof: projection matrix}
\begin{proof}
  We first note that
  $\Dnorm{Xw - v}^2 = \norm{D^{1/2}(Xw - v)}^2$.
  Then, by Theorem~\ref{thm: least squares}, it holds that
  $X \arg \min_{w \in \arg \min_w \Dnorm{Xw - v}^2} \norm{w}
    =X \arg \min_{w \in \arg \min_w \norm{D^{1/2}(Xw - v)}^2} \norm{w}
    =X(D^{1/2}X)^\dagger D^{1/2} v$.
  Hence, we have $\Pi = X(D^{1/2}X)^\dagger D^{1/2}$ by~\eqref{eq new projection}.
\end{proof}

\subsection{Proof of Lemma~\ref{lemma: contraction operator}}
\label{proof: contraction operator}
\begin{proof}
  Lemma 4 of \citet{tsitsiklis1997analysis} proves that $\bop$ is a contraction mapping w.r.t. $\Dnorm{\cdot}$ under Assumption~\ref{assumption: ergodicity}.
  Next, we show that $\Pi$ is nonexpansive w.r.t. $\Dnorm{\cdot}$.
  Define $Z \doteq D^{1/2} X$.
  Then, by Lemma~\ref{lemma: AA norm}, we have
  $\Dnorm{\Pi v} 
   =\Dnorm{X (D^{1/2} X)^\dagger D^{1/2} v}
   =\norm{Z Z^\dagger D^{1/2} v}
   \le\norm{Z Z^\dagger} \norm{D^{1/2} v}
   \le\norm{D^{1/2} v}
   =\Dnorm{v}$.
  It then follows immediately that $\bop \Pi$ is a contraction w.r.t. $\Dnorm{\cdot}$.
\end{proof}

\subsection{Proof of Lemma~\ref{lemma: value equivalence for TD fixed points}}
\label{proof: value equivalence for TD fixed points}
\begin{proof}
  Our proof relies on the fact that $D\qty(\gamma P - I)$ is negative definite~\citep{sutton2016emphatic}.
  For $w, w' \in \fW_*$, we have
  \begin{align}
    \textstyle A w + b - \qty(A w' + b) &= 0\\
    \textstyle A(w - w') &= 0 \\
    \textstyle X^\top D (\gamma P - I) X(w - w') &= 0\\
    \textstyle (w - w')^\top X^\top D (\gamma P - I) X(w - w') &= 0\\
    \textstyle X(w - w') &= 0 \explain{$D\qty(\gamma P - I)$ negative definite}\\
    \textstyle Xw & = Xw'.
  \end{align}
\end{proof}

\subsection{Proof of Theorem~\ref{thm: td fixed point and mspbe}}
\label{proof: td fixed point and mspbe}
\begin{proof}
  We begin with the direction $\Pi \bop X w = Xw \implies Aw + b = 0$.
  Define $\Omega \doteq \qty{w \mid\Pi \bop Xw = Xw}$.
  Suppose $w_* \in \Omega$, we have 
  \begin{align}
    \textstyle \Pi \bop X w_* &= X w_*\\
    \textstyle X (D^{1/2} X)^\dagger D^{1/2} \bop X w_* &= X w_*\\
    \textstyle D^{1/2} X (D^{1/2} X)^\dagger D^{1/2} \bop X w_* &= D^{1/2} X w_*\\
    \textstyle D^{1/2} X (D^{1/2} X)^\dagger D^{1/2} \bop X w_* &= 
    \textstyle D^{1/2} X (D^{1/2} X)^\dagger D^{1/2} X w_* \explain{$A A^\dagger A = A$}\\
    \textstyle D^{1/2} X (D^{1/2} X)^\dagger D^{1/2} (\bop X w_* - X w_*) &= 0\\
    \textstyle (D^{1/2} X)^\top D^{1/2} X (D^{1/2} X)^\dagger D^{1/2} (\bop X w_* - X w_*) &= 0\\
    \textstyle (D^{1/2} X)^\top D^{1/2} (\bop X w_* - X w_*) &= 0 \explain{$A^\top A A^\dagger = A^\top$}\\
    \textstyle X^\top D (r_\pi + \gamma P X w_* - X w_*) &= 0\\
    \textstyle X^\top D (\gamma P - I) X w_* + X^\top D r_\pi &= 0\\
    \textstyle Aw_* + b &= 0.
  \end{align}
  We now have $\Pi \bop X w = Xw \implies Aw + b = 0$.
  Next, we proceed to proving the other direction, i.e., $\Pi\bop X w = X w \impliedby A w + b = 0 $.
  In view of~\eqref{eq w star},
  there exists at least one $w_*$ such that $w_* \in \Omega$.
  The proof in the direction of $\implies$ then confirms that $w_* \in \fW_*$.
  Let $w$ be any weight in $\fW_*$.
  Then, Lemma~\ref{lemma: value equivalence for TD fixed points} implies that
  $Xw = Xw_* = v_*$.
  In view of~\eqref{eq v star},
  this means $w \in \Omega$.
  So, we have now proved that $w \in \fW_* \implies w \in \Omega$,
  which completes the proof.
\end{proof}
\bibliography{bibliography.bib}

\newpage

\end{document}